\documentclass[11pt]{airesearch}

\usepackage{fullpage}
\usepackage{wrapfig}
\usepackage{amsfonts} % For \mathbb fonts
\usepackage{mathrsfs} % For \mathscr
\usetikzlibrary{arrows.meta, backgrounds, positioning, fit, shapes.geometric, shadows}
\usepackage[all]{xy}

\usepackage{algorithm}
\usepackage{algpseudocode}  % Provides the algorithmic commands like \Require, \State, etc.

\usepackage{fancyhdr}
\newcommand{\projectpage}[1]{%
  \begin{center}
    \small
    \vspace{-1.25ex}
    \texttt{Project Page}: \url{#1}
  \end{center}
}

\newcommand{\id}{\mathrm{id}} % Identity morphism
\newcommand{\comp}{\circ} % Composition
\newcommand{\Prob}{\mathcal{P}} % Space of Probability measures
\newcommand{\SigmaAlg}[1]{\mathscr{B}(#1)} % Borel sigma-algebra for X (using script B) - Assuming standard Borel spaces primarily
\newcommand{\Borel}[1]{\mathscr{B}(#1)} % Borel sigma-algebra
\newcommand{\Expect}{\mathbb{E}} % Expectation
\newcommand{\Ind}{\mathbf{1}} % Indicator function
\newcommand{\defeq}{\coloneqq} % Definition equality
\newcommand{\dd}{\mathrm{d}} % Differential d for integration
\newcommand{\deltaDirac}[1]{\delta_{#1}} % Dirac measure
\newcommand{\powerset}[1]{\mathcal{P}(#1)} % Power set (using mathcal P)

\newcommand{\cat}[1]{\texttt{#1}} % Category font
\newcommand{\Stoch}{\cat{Stoch}} % Category of Standard Borel Spaces and Kernels
\newcommand{\MC}{\mathcal{C}} % Generic Markov Category (using mathcal)
\newcommand{\tens}{\otimes} % Tensor product
\newcommand{\unit}{\mathsf{I}} % Monoidal unit (typically a singleton space)
\newcommand{\swap}{\sigma} % Swap morphism
\newcommand{\copyop}{\Delta} % Copy map (diagonal map / comult)
\newcommand{\delop}{!} % Discard map (unique map to unit)

\newcommand{\KLDiv}{D_{\mathrm{KL}}} % KL Divergence
\newcommand{\TVDist}{d_{\mathrm{TV}}} % Total Variation Distance
\newcommand{\RenyiDiv}[1]{D_{#1}} % Renyi Divergence
\newcommand{\FD}{D_f} % f-Divergence
\newcommand{\Info}{I} % Mutual Information
\newcommand{\CatEnt}{\mathcal{H}^{\mathrm{cat}}} % pointwise categorical entropy (function X->R_+)
\newcommand{\AvgCatEnt}{\overline{\mathcal{H}}^{\mathrm{cat}}} % state-averaged categorical entropy (scalar)
\newcommand{\ShannonEntropy}{H} % Shannon entropy
\newcommand{\CategoricalEntropy}{\CatEnt}

\newcommand{\VocabSpaceMeas}{(\VV, \powerset{\VV})} % Vocabulary measurable space (finite set)
\newcommand{\SeqSpace}{\VV^*} % Space of finite sequences (Kleene star)
\newcommand{\SeqSpaceMeas}{(\SeqSpace, \SigmaAlg{\SeqSpace})} % Measurable space of sequences (assuming suitable sigma-alg making it standard Borel)
\newcommand{\ContextSeqSpace}{\VV^*} % Space of context sequences (= SeqSpace)
\newcommand{\ContextSeqSpaceMeas}{\SeqSpaceMeas} % Measurable space for context sequences
\newcommand{\RepSpace}{\mathcal{H}} % Representation space (e.g., R^d)
\newcommand{\RepSpaceMeas}{(\RepSpace, \SigmaAlg{\RepSpace})} % Representation measurable space
\newcommand{\ContextRepSpace}{\mathcal{H}_{\mathrm{seq\_emb}}} % Intermediate sequence embedding space (e.g., space of sequences of vectors)
\newcommand{\ContextRepSpaceMeas}{(\ContextRepSpace, \SigmaAlg{\ContextRepSpace})} % Measurable space for sequence embeddings (assuming standard Borel)
\newcommand{\Emb}{\mathcal{E}} % Embedding function (token to vector)
\newcommand{\EmbLayer}{f_{\mathrm{emb}}} % Embedding Layer function
\newcommand{\Backbone}{f_{\mathrm{bb}}} % Backbone function (e.g., Transformer layers)
\newcommand{\LMHead}{f_{\mathrm{head}}} % LM Head function (pre-softmax linear layer)
\newcommand{\Params}{\theta} % Model parameters
\newcommand{\KernelEmb}{k_{\mathrm{emb}}} % Embedding Layer kernel
\newcommand{\KernelBB}{k_{\mathrm{bb}}} % Backbone kernel
\newcommand{\KernelLMHead}{k_{\mathrm{head}}} % LM Head kernel
\DeclareMathOperator{\softmax}{softmax} % Softmax function

\ifdefined\final
\usepackage[disable]{todonotes}
\else
\usepackage[textsize=tiny]{todonotes}
\fi
\title{\huge \bfseries \sffamily A Markov Categorical Framework \\for Language Modeling}
\author{\textbf{Yifan Zhang}\\[1.5mm] Princeton University\\[0.5mm] \texttt{yifzhang@princeton.edu}}
\date{}

\begin{document}
\maketitle

%%%%%%%%%%%%%%%%%%%%%%%%%%%%%%
%%% Abstract
\begin{abstract}
Autoregressive language models achieve remarkable performance, yet a unified theory explaining their internal mechanisms, how training shapes representations, and why these representations support complex behavior remains incomplete. We introduce an analytical framework that models the single-step generation process as a composition of information-processing stages using the language of Markov categories. This compositional perspective connects three aspects of language modeling that are often studied separately: the training objective, the geometry of the learned representation space, and practical model capabilities. First, our framework gives an information-theoretic rationale for parallel drafting methods such as speculative decoding by quantifying the information surplus a hidden state contains about future tokens beyond the immediate next one. Second, we clarify how the standard negative log-likelihood (NLL) objective learns not only a most likely next token, but also the data's intrinsic conditional uncertainty, formalized through categorical entropy. Our main spectral result is conditional: for a linear-softmax head with bounded output features, a calibrated quadratic upper-bound surrogate to NLL induces, after whitening or variance normalization, a generalized CCA/eigenproblem aligning representation directions with predictive prototypes. This gives a compositional lens for understanding how information flows through a model and how likelihood training can shape its internal geometry.
\end{abstract}

\projectpage{https://github.com/yifanzhang-pro/lm-theory}

%%%%%%%%%%%%%%%%%%%%%%%%%%%%%%
%%% Introduction
\section{Introduction}
\label{sec:introduction}

\begin{figure*}[ht!]
\centering
\resizebox{0.98\textwidth}{!}{
\begin{tikzpicture}[
  node distance=1cm and 1.5cm,
  % STYLES
  main_flow/.style={rectangle, draw, fill=blue!10, rounded corners, drop shadow, minimum height=3.5em, text centered, text width=5cm, font=\small},
  seq_block/.style={rectangle, draw, fill=blue!20, rounded corners, drop shadow, minimum size=2em},
  concept/.style={rectangle, draw, fill=green!10, rounded corners, text width=6cm, align=center, font=\small},
  analysis/.style={rectangle, draw, fill=violet!15, rounded corners, text width=6cm, align=center, font=\small},
  % ARROW STYLES
  arrow/.style={-Stealth, very thick, draw=black!80},
  implies/.style={->, thick, draw=black!70, rounded corners},
  shapes/.style={-Stealth, ultra thick, draw=green!80!black, dotted},
  enables/.style={-Stealth, ultra thick, draw=blue!40!white, dotted},
  main_claim/.style={-Stealth, ultra thick, draw=orange!80!white, rounded corners=15pt, line width=1pt}
]

% 1. THE CENTRAL MODEL
\node[main_flow] (context) at (-11, 0) {$(\mathcal{V}^*)$\\ Context $\mathbf{w}_{<t}$};

% --- Sequence Embedding Blocks ---
\node[seq_block] (seq2) at (-3.7, 0) {};
\node[seq_block, left=0.1cm of seq2] (seq1) {};
\node[seq_block, right=0.1cm of seq2] (seq3) {};
\node[draw=black!40, dashed, rounded corners, inner xsep=1em, inner ysep=0.8em,
  fit=(seq1)(seq3),
  label={[label distance=-0.6em, align=center]below:{($\mathcal{H}_{\mathrm{seq\_emb}})$\\Sequence Embeddings}}]
  (emb_seq_group) {};

% --- MODIFIED: Hidden State Blocks (replaces the single h_state) ---
\node[seq_block] (h_state2) at (3.7, 0) {}; % Place the middle one first
\node[seq_block, left=0.1cm of h_state2] (h_state1) {};
\node[seq_block, right=0.1cm of h_state2] (h_state3) {};
% Group the new hidden states with a fit node and a label
\node[draw=black!40, dashed, rounded corners, inner xsep=1em, inner ysep=0.8em,
  fit=(h_state1)(h_state3),
  label={[label distance=-0.6em, align=center]below:{$(\mathcal{H})$\\Hidden States $h_{t,i}$}}]
  (h_state_group) {};

% --- Final Vocab Node    ---
\node[main_flow] (vocab) at (11, 0) {$(\mathcal{V})$\\ Next Token Dist.};

% --- ARROWS    ---
% From Context to Embedding
\draw[arrow] (context.east) -- node[above, yshift=2mm] {$k_{\mathrm{emb}}$} (seq1.west);
% Internal embedding arrows
\draw[-Stealth, thick, black!60] (seq1.east) -- (seq2.west);
\draw[-Stealth, thick, black!60] (seq2.east) -- (seq3.west);
% From LAST embedding to FIRST hidden state
\draw[arrow] (seq3.east) -- node[above] {$k_{\mathrm{bb}}$} (h_state1.west);
% Internal hidden state arrows
\draw[-Stealth, thick, black!60] (h_state1.east) -- (h_state2.west);
\draw[-Stealth, thick, black!60] (h_state2.east) -- (h_state3.west);
% From LAST hidden state to vocab
\draw[arrow] (h_state3.east) -- node[above] {$k_{\mathrm{head}}$} (vocab.west);

% 2. THE CAUSE: ``Why it Works'' (Top Box)
\node[concept, minimum height=7em] (spectral) at (-6, 5) {
  \textbf{Implicit Spectral Contrastive Learning (\S7)}\\
  NLL sculpts a geometrically structured representation space.
};
\node[concept, minimum height=7em] (stochasticity) at (6, 5) {
  \textbf{Learning Intrinsic Stochasticity (\S5)}\\
  The model learns the data's inherent randomness.
};
\node[draw, fit=(spectral)(stochasticity), rounded corners, inner sep=10pt, label={[font=\bfseries]above:Why it Works: The NLL Objective}] (top_box) {};

% MODIFIED: Conceptual arrow points to the new hidden state group
\draw[shapes] (top_box.south) -- (emb_seq_group.north);
\draw[shapes] (top_box.south) -- (h_state_group.north);

% 3. THE EFFECT: ``What it Reveals'' (Bottom Box)
\node[analysis, minimum height=7em] (infogeo) at (-6, -5) {\textbf{Information Geometry of Representations (\S6)}\\The pullback metric $g^*$ on $\mathcal{H}$ measures predictive sensitivity.
};
\node[analysis, minimum height=7em] (specdec) at (6, -5) {\textbf{Rationale for Speculative Decoding (\S4.1)}\\Quantifies the information surplus in the hidden state $h_t$.
};
\node[draw, fit=(infogeo)(specdec), rounded corners, inner sep=10pt, label={[font=\bfseries]below:What it Reveals: Analysis \& Applications}] (bottom_box) {};

% MODIFIED: Conceptual arrows point from the new hidden state group
\draw[enables] (h_state_group.south) -- (infogeo.north);
\draw[enables] (h_state_group.south) -- (specdec.north);

% 4. THE MAIN THEORETICAL CLAIM
\draw[main_claim] (spectral.south) .. controls (-6.5, -1.5) and (-6.5, -3.5) .. (infogeo.north);

\end{tikzpicture}
} % end resizebox
\caption{A conceptual overview of our framework. \textbf{Center:} The core thesis models the Autoregressive generation step as a composition of Markov kernels $k_{\mathrm{gen}} = k_{\mathrm{head}} \comp k_{\mathrm{bb}} \comp k_{\mathrm{emb}}$ in the category $\Stoch$. This separates the deterministic context encoding ($k_{\mathrm{emb}}, k_{\mathrm{bb}}$) from the probabilistic output \emph{kernel} $k_{\mathrm{head}}$, which is parameterized by a deterministic map $g_{\mathrm{head}}\!:\mathcal H\to\Delta$. \textbf{Top:} This compositional lens reveals the deeper mechanisms of the NLL objective, which we re-frame as minimizing the average KL divergence between the model and true data kernels. Under additional constraints satisfied by linear-softmax LM heads (see \S\ref{sec:repr_learning_theory}), we show a conditional spectral connection with a predictive-similarity operator; in all cases, NLL compels the model to learn intrinsic conditional stochasticity (via categorical entropy). \textbf{Bottom:} Pulling back the Fisher–Rao metric endows $\mathcal{H}$ with an information geometry that quantifies predictive sensitivity and clarifies the information surplus used by speculative decoding.}
\label{fig:overview}
\end{figure*}
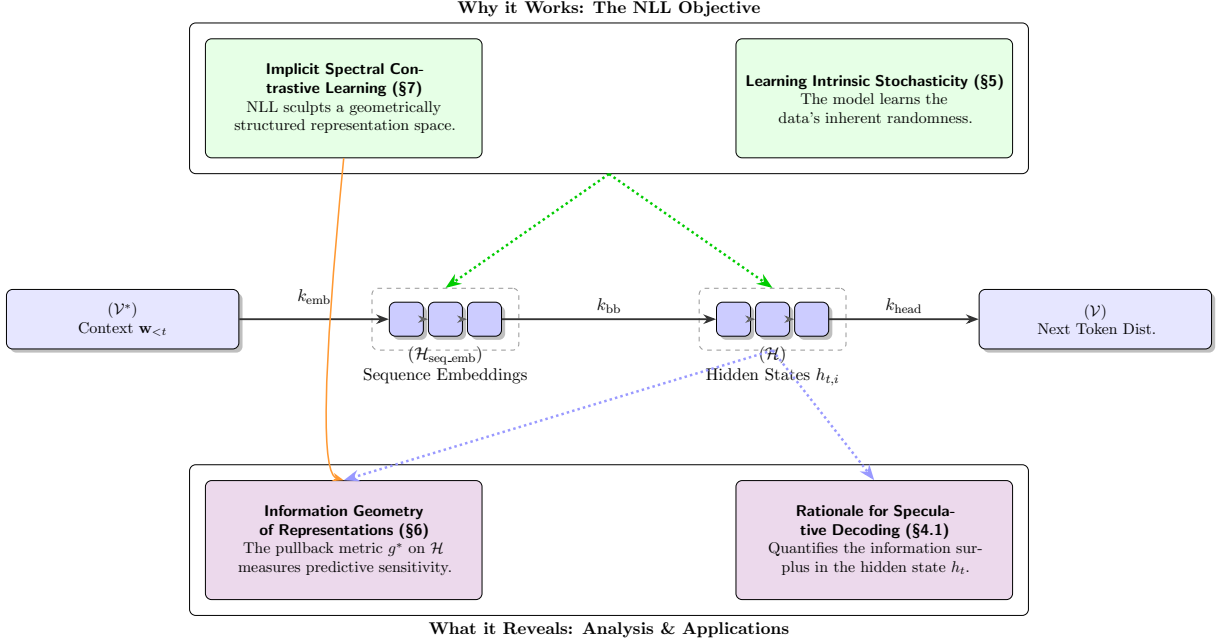

Autoregressive language models (AR LMs), particularly those based on the Transformer architecture \citep{Vaswani2017, Radford2019, Brown2020}, have achieved remarkable success, defining the state-of-the-art in natural language generation and demonstrating impressive few-shot learning capabilities. These models operate by sequentially predicting the next token in a sequence based on the preceding context. Formally, given a sequence $\mathbf{w} = w_1 \dots w_L$ with tokens $w_i$ from a finite vocabulary $\VV$, the model learns a parameterized probability distribution $P_\Params$ that factorizes as:
\begin{align}
P_\Params(\mathbf{w}) = \prod_{t=1}^L P_\Params(w_t | \mathbf{w}_{<t}), \label{eq:ar_factorization}
\end{align}
where $\mathbf{w}_{<t} \defeq w_1 \dots w_{t-1}$ is the context sequence, and $\Params$ denotes the model parameters, typically optimized by minimizing the negative log-likelihood (NLL) on vast text corpora. The core computational step is the mapping from a context $\mathbf{w}_{<t}$ to the conditional probability distribution $P_\Params(\cdot | \mathbf{w}_{<t})$ over $\VV$ for the next token $w_t$.

Despite their empirical triumphs, a deep theoretical understanding of their internal mechanisms remains incomplete~\citep{Manning2020, Elhage2021Circuits, Yuan2022PowerFoundation}. Current analysis often relies on empirical probes \citep{Hewitt2019StructuralProbe} or studies of specific components like attention heads \citep{Olsson2022InContext}. While insightful, these methods can be fragmented and often lack a unified mathematical language to describe the model's compositional and stochastic nature as a whole. A central goal here is not to introduce yet another isolated tool, but to connect well-established tools within a single, compositional language.

Another critical challenge is improving the slow, sequential nature of AR generation. Recent advances in speculative decoding, such as EAGLE~\citep{li2024eagle}, have achieved significant speedups by predicting multiple tokens in parallel, suggesting that the final hidden state $h_t$ contains far more information than is needed for predicting only the single next token $w_t$. However, a formal understanding of this information surplus is lacking.

This paper addresses this gap by introducing a unifying analytical framework for AR LMs. Our central thesis is that the language of Markov Categories (MCs)~\citep{ChoJacobs2019, Fritz2020MC} provides a natural mathematical setting for connecting several concepts that are usually discussed separately: information flow through the model's components, the geometry of the learned representation space, and the structural effects of the NLL training objective.

While many individual mathematical tools we employ, such as the pullback of the Fisher-Rao metric or the connection between NLL and KL divergence, are well-established, our primary contribution is their novel synthesis and application to dissect AR LMs. The originality of this work lies in using the category $\Stoch$ to formally model compositional information flow, leading to new insights uniquely enabled by this perspective. Unlike information-theoretic analyses that treat models as monolithic black boxes analyzing external behavior (e.g., the entropy of the output sequence), our framework uses categorical information theory to analyze the internal transformations and the learned geometry of the representation space at each stage of processing.

This paper introduces an analytical framework focused on the internal mechanics of the AR generation step $\mathbf{w}_{<t} \mapsto P_\Params(\cdot | \mathbf{w}_{<t})$. We leverage the category $\Stoch$, a canonical MC whose objects are standard Borel spaces (like the continuous representation space $\RepSpace \cong \mathbb{R}^{d_{\mathrm{model}}}$) and whose morphisms are Markov kernels \citep{Kallenberg2002, Fritz2020MC}. We formalize the AR generation step as a composite kernel in $\Stoch$:
\begin{equation}
k_{\mathrm{gen}, \theta} \defeq \KernelLMHead \comp \KernelBB \comp \KernelEmb : \ContextSeqSpaceMeas \to \VocabSpaceMeas.\label{eq:k_gen_intro}
\end{equation}
Here, $\KernelEmb$ and $\KernelBB$ represent the typically deterministic context embedding and backbone transformations that produce the final hidden state $h_t \in \RepSpace$, while $\KernelLMHead$ is the \emph{stochastic} kernel induced by a deterministic head map $g_{\mathrm{head}}:\RepSpace\to\Prob(\VV)$, $h\mapsto p_h$; randomness enters only through sampling $W_t\sim p_{h_t}$.

A crucial aspect of our framework is enriching $\Stoch$ with a statistical divergence $D$ (e.g., $\KLDiv$) \citep{BaezFongPollard2016, Perrone2022Ent, Perrone2023Geo}. This allows for defining intrinsic, categorical information measures like entropy $\CategoricalEntropy_D$ and mutual information $\Info_D$ \citep{Perrone2022Ent}, which automatically satisfy the Data Processing Inequality (DPI). Leveraging this unified framework, this paper makes the following contributions:

\begin{enumerate}
\item We formally model the AR generation step as a composite kernel, $k_{\mathrm{gen}, \theta} = \KernelLMHead \comp \KernelBB \comp \KernelEmb$. This compositional structure is a powerful tool for reasoning about how information is transformed, preserved, or lost at each distinct stage of processing.
\item We use categorical information measures, the chain rule for mutual information, and DPI-compatible processing bounds to quantify the information surplus used by methods like EAGLE~\citep{li2024eagle}. This surplus is the information in a hidden state $H_t$ about multiple future tokens beyond the immediate next token. It is an information budget for parallel drafting, not by itself a guarantee of realized speedup.
\item We show how the MC framework unifies three critical interpretations of NLL training under one theoretical roof:
\begin{itemize}
\item NLL as KL minimization, equivalent to optimal source coding.
\item NLL forces the model to learn the data's inherent randomness, a process we formalize using categorical entropy. We show that optimizing NLL implies that the model's learned stochasticity converges to that of the data (Theorem \ref{thm:entropy_convergence}).
\item Under a linear-softmax head with bounded output features, we show that a calibrated quadratic upper-bound surrogate to NLL yields a regression-to-predictive-prototypes objective. After whitening or variance normalization, the associated alignment problem becomes a generalized CCA/eigenproblem. This gives a precise conditional sense in which likelihood training can align representation directions with a predictive-similarity operator (Theorem \ref{thm:nll_spectral_cca}).
\end{itemize}
\end{enumerate}

By formalizing how information is transformed (\Cref{sec:markov-categorical-metrics}), how predictive sensitivity is encoded in representation geometry (\Cref{sec:infogeo}), and how the NLL objective implicitly structures representations (\Cref{sec:repr_learning_theory}), we can move towards more principled approaches to model design, interpretation, and control.

The paper is organized as follows. \Cref{sec:background_mc,sec:ar_model_mc} introduce the MC framework and our compositional model of LMs. \Cref{sec:markov-categorical-metrics} uses the framework to analyze information flow, providing a rationale for speculative decoding. \Cref{sec:pretraining_compression_entropy} connects the NLL objective to learning the data's intrinsic stochasticity. \Cref{sec:infogeo,sec:repr_learning_theory} present our main theoretical result, showing how NLL performs implicit spectral learning by shaping the geometry of the representation space. \Cref{sec:related_work,sec:conclusion} discuss related work and then conclude.

\section{Background}
\label{sec:background_mc}

This section reviews the essential mathematical concepts forming the foundation of our framework: the definition of Markov Categories and the specific category $\Stoch$, followed by the enrichment of $\Stoch$ with statistical divergences leading to categorical information measures.

\subsection{Markov Categories and \texorpdfstring{$\Stoch$}{Stoch}}

Markov Categories provide an axiomatic framework for probability and stochastic processes using category theory \citep{Fritz2020MC}.

\begin{definition}[Markov Category \citep{Fritz2020MC}]
A Markov category $(\MC,\tens,\unit)$ is a symmetric monoidal category in which each object $X$ is equipped with a commutative comonoid structure $(\copyop_X:X\to X\tens X,\delop_X:X\to\unit)$, compatible with the monoidal product. The discard maps are natural: for every morphism $f:X\to Y$, $\delop_Y\comp f=\delop_X$. Equivalently, in the causal case relevant here, the monoidal unit $\unit$ is terminal. The copy maps are \emph{not} natural for arbitrary stochastic morphisms; morphisms satisfying $\copyop_Y\comp f=(f\tens f)\comp\copyop_X$ are the deterministic morphisms.
\end{definition}

Morphisms $k:X\to Y$ are interpreted as stochastic processes. Composition $h\comp k$ is sequential processing, while $k\tens h$ is parallel processing. The comonoid maps $\copyop_X$ (copy) and $\delop_X$ (discard) abstractly model duplication and deletion of information. Copying before and after a genuinely stochastic map need not agree; this mismatch is exactly what categorical entropy measures. The causality axiom enforces probability normalization ($\int k(x,\dd y)=1$) in concrete examples like $\Stoch$. States (probability distributions) on an object $X$ are represented as morphisms $p:\unit\to X$.

The key example for our purposes is the category $\Stoch$.

\begin{definition}[Category $\Stoch$ \citep{Fritz2020MC, Perrone2022Ent}]
The Markov category $\Stoch$ is defined by:
\begin{itemize}
\item \textbf{Objects}: Standard Borel spaces $(X, \SigmaAlg{X})$. These are general measure spaces that include finite sets (like a vocabulary $\VV$), countable sets, and continuous spaces like Euclidean space $\mathbb{R}^d$ or other Polish spaces. This ensures the framework can handle both discrete tokens and continuous representations. The monoidal unit $\unit$ is a singleton space $(\{\star\}, \{\emptyset, \{\star\}\})$.
\item \textbf{Morphisms}: Markov kernels $k: X \to Y$. A map $k: X \times \SigmaAlg{Y} \to [0, 1]$ where $k(x, \cdot)$ is a probability measure on $Y$ for each $x \in X$, and $k(\cdot, A)$ is a measurable function on $X$ for each $A \in \SigmaAlg{Y}$.
\item \textbf{Composition}: Given $k: X \to Y$ and $h: Y \to Z$, the composite $h \comp k: X \to Z$ is $(h \comp k)(x, C) \defeq \int_Y h(y, C) \, k(x, \dd y)$ (Chapman-Kolmogorov). Identity $\id_X(x, A) = \deltaDirac{x}(A)$.
\item \textbf{Monoidal Product ($\tens$)}: Product space $(X \times Y, \SigmaAlg{X} \tens \SigmaAlg{Y})$ with the product $\sigma$-algebra. Product kernel $(k \tens h)((x, y), \cdot) \defeq k(x, \cdot) \tens h(y, \cdot)$ (product measure).
\item \textbf{Symmetry}: Swap map $\swap_{X,Y}: X \tens Y \to Y \tens X$ is $\swap_{X,Y}((x, y), \cdot) = \deltaDirac{(y, x)}$.
\item \textbf{Comonoid Structure}: Copy $\copyop_X: X \to X \tens X$ is $\copyop_X(x, \cdot) = \deltaDirac{(x, x)}$. Discard $\delop_X: X \to \unit$ maps to the unique point measure on $\unit$, $\delop_X(x, \{\star\}) = 1$.
\item \textbf{Causality}: $\unit$ is terminal, $\delop_Y \comp k = \delop_X$ holds, reflecting probability normalization.
\end{itemize}
\end{definition}
\begin{remark}[Interpretation]
In $\Stoch$, objects represent the types of random outcomes (e.g., sequences, vectors, tokens). Morphisms represent stochastic processes or channels mapping inputs to probability distributions over outputs. Deterministic functions $f: X \to Y$ correspond to deterministic kernels $k_f(x, \cdot) = \deltaDirac{f(x)}$. States $p: \unit \to X$ correspond bijectively to probability measures $\mu_p \in \Prob(X)$ via $\mu_p(A) = p(\star, A)$. Marginalization arises from discarding information, e.g., for a joint state $p: \unit \to X \tens Y$, the $X$-marginal is $p_X = (\id_X \tens \delop_Y) \comp p$.
\end{remark}

\begin{lemma}[Standard Borel and measurability]
\label{lem:standard-borel}
The spaces $\VV^*$ (countable disjoint union of finite products), $\ContextRepSpace$ (countable disjoint union of Euclidean products), $\RepSpace\simeq\mathbb R^{d_{\mathrm{model}}}$, and $\VV$ (finite) are standard Borel. If $\EmbLayer,\Backbone$ are Borel-measurable, then the induced deterministic kernels $\KernelEmb,\KernelBB$ are morphisms in $\Stoch$.
\end{lemma}
\begin{proof}[Sketch]
Countable disjoint unions of Polish spaces are standard Borel; products/sums preserve standard Borelness. Deterministic kernels defined by Borel maps are measurable morphisms in $\Stoch$.
\end{proof}

\subsection{Divergence Enrichment and Categorical Information Measures}

The structure of $\Stoch$ is particularly powerful when enriched with a statistical divergence $D$, quantifying the dissimilarity between probability measures (states) $p, q: \unit \to X$, written $D_X(p \| q)$ \citep{Perrone2022Ent}. Examples include KL divergence ($\KLDiv$), Total Variation ($\TVDist$), Rényi divergences ($\RenyiDiv{\alpha}$), and the broad class of $f$-divergences ($\FD$) \citep{AmariNagaoka2000, Nowozin2016fGAN}.

A fundamental property linking divergences and Markov kernels is the Data Processing Inequality (DPI), which holds for most standard divergences, including $f$-divergences and the usual Rényi divergences in their standard DPI ranges.

\begin{theorem}[Data Processing Inequality (DPI)]
Let $D$ be a statistical divergence satisfying the DPI. For any Markov kernel $k: X \to Y$ in $\Stoch$ and any pair of states $p, q: \unit \to X$:
\begin{equation}
D_Y(k \comp p \| k \comp q) \le D_X(p \| q)
\label{eq:dpi_states_revised}
\end{equation}
Processing through the channel $k$ cannot increase the $D$-divergence between the distributions.
\end{theorem}

\noindent\textit{Remark.} For Rényi divergences, standard DPI statements cover $\alpha\in(0,\infty]$ under the usual absolute-continuity conventions, with KL recovered as the limit $\alpha\to1$ and max-divergence recovered as $\alpha\to\infty$ when finite. We use $\alpha\in(0,\infty)$ as the default and invoke endpoint cases only when the required absolute-continuity conditions are explicit.

\vspace{-0.5ex}
\begin{assumption}[Standing assumptions]\label{assump:standing}
We work with: (i) objects that are standard Borel; (ii) Borel-measurable $\EmbLayer,\Backbone$, hence deterministic kernels in $\Stoch$; (iii) finite vocabulary $\VV$; (iv) a deterministic parameterization $g_{\mathrm{head}}:\RepSpace\to\Prob(\VV)$ that is differentiable on a full-measure subset under $p_{H_t}$; (v) finite second moments $\mathbb E\|H_t\|^2<\infty$; (vi) where invoked (e.g., \Cref{thm:entropy_convergence}), well-posed weak convergence of the relevant joint laws; and (vii) unless stated otherwise, analysis restricted to the \emph{interior} of the simplex (i.e., $p_h(w)>0$ for all $w$) on a full-measure set, so that Fisher–Rao and second-order KL expansions are valid.
For near-boundary behavior, we work on a high-probability subset where $p_h(w)\ge \delta>0$ and carry $\delta$-dependent constants.
\end{assumption}

\paragraph{Preliminaries: NLL-KL equivalence and distance conventions.}
We recall the classical identity that minimizing cross-entropy is equivalent to minimizing the average KL divergence between data and model conditionals. Throughout, we use the Hellinger distance with the single global convention
\[
\;d_H^2(p,q)\;\defeq\;1-\sum_{w\in\VV}\sqrt{p(w)q(w)}
\;=\;\tfrac12\sum_{w\in\VV}\big(\sqrt{p(w)}-\sqrt{q(w)}\big)^2\,,
\]
On finite alphabets, we will use the bounds
\[
d_H^2(p,q)\ \le\ 1-e^{-\KLDiv(p\|q)/2}\ \le\ \tfrac12\,\KLDiv(p\|q).
\]
(Equivalently, \(1-\sum_w\sqrt{p(w)q(w)}\le 1-e^{-\KLDiv(p\|q)/2}\le \tfrac12\,\KLDiv(p\|q)\).)
We use $\mathrm{TV}(p,q)=\tfrac12\|p-q\|_1$ when needed, with $\mathrm{TV}^2(p,q)\le \tfrac12\,\KLDiv(p\|q)$.  This convention is used consistently in all subsequent sections.

We provide a full proof of the NLL-KL equivalence in Appendix~\ref{app:proof_nll_kl_equivalence}.

\begin{theorem}[NLL Minimization as Average KL Minimization]
\label{thm:nll_kl_equivalence}
It is a well-known result in information theory that minimizing the cross-entropy loss $L_{\mathrm{CE}}(\theta)$ is equivalent to minimizing the average KL divergence between the true and model conditional distributions. We state it here in the language of our framework to ground the subsequent analysis.
\begin{equation}
\mathcal{L}_{\mathrm{KL}}(\theta)
\defeq
\Expect_{\mathbf{w}_{<t} \sim p_{W_{<t}}}\!\Big[
\KLDiv\!\big( k_{\text{data}}(\mathbf{w}_{<t}, \cdot) \,\|\, k_{\mathrm{gen}, \theta}(\mathbf{w}_{<t}, \cdot) \big)
\Big],
\qquad
\argmin_{\theta} L_{\mathrm{CE}}(\theta) = \argmin_{\theta} \mathcal{L}_{\mathrm{KL}}(\theta).
\label{eq:kl_objective_kernels_revisited}
\end{equation}
The expectation is taken over contexts $\mathbf{w}_{<t}$ drawn according to the data's marginal context distribution $p_{W_{<t}}$. The minimum value of $\mathcal{L}_{\mathrm{KL}}(\theta)$ is non-negative. If the model class $\{ k_{\mathrm{gen}, \theta} \mid \theta \in \Theta \}$ is sufficiently expressive to contain $k_{\text{data}}$ (i.e., $k_{\text{data}} = k_{\mathrm{gen}, \theta_{\text{true}}}$ for some $\theta_{\text{true}} \in \Theta$), then the minimum value is 0, achieved if and only if $k_{\mathrm{gen}, \theta^*}(\mathbf{w}_{<t}, \cdot) = k_{\text{data}}(\mathbf{w}_{<t}, \cdot)$ for $p_{W_{<t}}$-almost every context $\mathbf{w}_{<t}$.
\end{theorem}

Perrone \citep{Perrone2022Ent} introduced categorical definitions of entropy and mutual information intrinsically tied to the divergence $D$ and the MC structure.

\begin{definition}[Categorical Entropy \citep{Perrone2022Ent}]
Let $(\Stoch, D)$ be enriched with a DPI-satisfying divergence $D$.
\begin{enumerate}
\item The \emph{pointwise categorical entropy} of a kernel $k: X \to Y$ is the function $\CatEnt_D(k):X\to\mathbb{R}_{\ge0}$ given by
\begin{equation}
\CatEnt_D(k)(x) \defeq D_{Y \tens Y} \!\left( \copyop_Y \comp k(x,\cdot) \;\middle\| \; (k \tens k) \comp \copyop_X(x,\cdot) \right).
\label{eq:cat_entropy_def_pointwise}
\end{equation}
Intuitively, it compares (1) applying $k$ then copying the output vs. (2) copying $x$ then applying $k$ independently. If $k$ is deterministic, $\CatEnt_D(k)(x)=0$.

\item Given a prior/state $p_X:\unit\to X$, the \emph{state-averaged categorical entropy} is the scalar
\begin{equation}
\AvgCatEnt_D(k;p_X)\;\defeq\;\Expect_{x\sim p_X}\big[\CatEnt_D(k)(x)\big] =\int_X \CatEnt_D(k)(x)\,p_X(\dd x).
\label{eq:cat_entropy_def_avg}
\end{equation}

\item The \emph{Categorical Mutual Information} for a joint state $p: \unit \to X \tens Y$ is $\Info_D(p) \defeq D_{X \tens Y} ( p \| p_X \tens p_Y )$.
\end{enumerate}
\end{definition}

\begin{remark}[Properties and Connections]
\label{rem:cat_mi_def}
When $D = \KLDiv$ and the output object $Y$ is \emph{finite} (in particular, $Y=\VV$), $\Info_{\KLDiv}(p)$ recovers Shannon mutual information. In this discrete setting, the pointwise categorical entropy satisfies $\CatEnt_{\KLDiv}(k)(x)=H\!\big(k(x,\cdot)\big)$, and the averaged quantity satisfies $\AvgCatEnt_{\KLDiv}(k;p_X)=\mathbb E_{x\sim p_X}\!\big[H(k(x,\cdot))\big]=H(Y\mid X)$. For non-atomic $Y$, the diagonal law $\copyop_Y\comp k(x,\cdot)$ is typically singular with respect to $(k\otimes k)\comp\copyop_X(x,\cdot)$, so $\CatEnt_{\KLDiv}(k)(x)$ is generally $+\infty$ unless $k(x,\cdot)$ is purely atomic; deterministic kernels are the zero-entropy special case.
\end{remark}

\section{Autoregressive Language Models as Composed Kernels}
\label{sec:ar_model_mc}

We now apply the Markov Category framework established in \Cref{sec:background_mc} to model Autoregressive language models. Specifically, we model the single-step generation mapping $\mathbf{w}_{<t} \mapsto P_\Params(\cdot | \mathbf{w}_{<t})$ as a composition of Markov kernels within the category $\Stoch$.

The relevant measurable spaces (objects in $\Stoch$) are:
\begin{itemize}
\item Input context space: $\ContextSeqSpaceMeas = (\VV^*, \SigmaAlg{\VV^*})$, where $\VV^*$ is the set of finite sequences over the vocabulary $\VV$, equipped with a suitable $\sigma$-algebra making it standard Borel (e.g., considering it as a disjoint union of finite products $\VV^n$).
\item Initial sequence representation space: $\ContextRepSpaceMeas = (\ContextRepSpace, \SigmaAlg{\ContextRepSpace})$, the space of initial vector sequences (e.g., $\bigcup_n (\mathbb{R}^{d_{\mathrm{model}}})^n$), also equipped with a standard Borel structure.
\item Final hidden state space: $\RepSpaceMeas = (\RepSpace, \SigmaAlg{\RepSpace})$, typically $(\mathbb{R}^{d_{\mathrm{model}}}, \Borel{\mathbb{R}^{d_{\mathrm{model}}}})$.
\item Output vocabulary space: $\VocabSpaceMeas = (\VV, \powerset{\VV})$, a finite measurable space.
\end{itemize}
Standard Borel spaces are chosen because they form a well-behaved class of measurable spaces (isomorphic to Borel subsets of Polish spaces) closed under countable products, sums, and containing standard examples like $\mathbb{R}^d$ and finite sets, ensuring measure-theoretic regularity \citep{Kallenberg2002}.

The generation process decomposes into three kernels (morphisms in $\Stoch$):

1. \textbf{Embedding Layer Kernel ($\KernelEmb: \ContextSeqSpaceMeas \to \ContextRepSpaceMeas$)}:
This kernel encapsulates the initial processing of the discrete input sequence $\mathbf{w}_{<t} \in \VV^*$. It typically involves applying a token embedding function $\Emb: \VV \to \mathbb{R}^{d_{\mathrm{model}}}$ to each token $w_i$ and potentially incorporating absolute positional encodings. Let $\EmbLayer: \VV^* \to \ContextRepSpace$ denote the overall deterministic function computing the initial sequence representation $E_{<t}$. Since this mapping is deterministic, the kernel $\KernelEmb$ is defined via the Dirac measure $\deltaDirac{\cdot}$:
\begin{equation}
\KernelEmb(\mathbf{w}_{<t}, A) \defeq \deltaDirac{\EmbLayer(\mathbf{w}_{<t})}(A) = \Ind_A(\EmbLayer(\mathbf{w}_{<t})), \quad \text{for } A \in \SigmaAlg{\ContextRepSpace}.
\end{equation}
This is a valid morphism in $\Stoch$.

2. \textbf{Backbone Transformation Kernel ($\KernelBB: \ContextRepSpaceMeas \to \RepSpaceMeas$)}:
This kernel represents the core computation, usually a deep neural network like a Transformer stack. Let $\Backbone: \ContextRepSpace \to \RepSpace$ be the function mapping the initial sequence representation $E_{<t}$ to the final hidden state $h_t \in \RepSpace$ (often the output vector at the last sequence position). This function incorporates complex operations like multi-head self-attention and feed-forward layers. Relative positional information, such as Rotary Position Embeddings (RoPE) \citep{Su2021RoFormer}, is implemented within the function $\Backbone$ by modifying attention computations based on token positions. Assuming the backbone computation is deterministic for a given $E_{<t}$ and parameters $\Params$, the kernel $\KernelBB$ is also deterministic:
\begin{equation}
\KernelBB(E_{<t}, B) \defeq \deltaDirac{\Backbone(E_{<t})}(B) = \Ind_B(\Backbone(E_{<t})), \quad \text{for } B \in \SigmaAlg{\RepSpace}.
\end{equation}
This is also a morphism in $\Stoch$.

3. \textbf{LM Head Kernel ($\KernelLMHead: \RepSpaceMeas \to \VocabSpaceMeas$)}:
This final step maps the summary hidden state $h_t \in \RepSpace$ to a probability distribution over the finite vocabulary $\VV$. Typically, $h_t$ is passed through a linear layer ($\LMHead: \RepSpace \to \mathbb{R}^{|\VV|}$) producing logits $\mathbf{z} = \LMHead(h_t)$, followed by the $\softmax$ function: $P(w | h_t) = [\softmax(\mathbf{z})]_w$. This defines a Markov kernel induced by a deterministic map into the simplex:
\begin{equation}
\KernelLMHead(h, A) \defeq \sum_{w \in A} [\softmax(\LMHead(h))]_w \quad \text{for } h \in \RepSpace, A \subseteq \VV.
\end{equation}
This kernel maps each point $h$ in the representation space to a probability measure on the discrete space $\VV$, satisfying the required measurability conditions. It is a generally non-deterministic morphism in $\Stoch$ that is parameterized by a deterministic map $g_{\mathrm{head}}:\RepSpace\to\Prob(\VV)$.

\begin{remark}[Deterministic parameterization vs.\ stochastic kernel]
The head is best seen as a deterministic map $g_{\mathrm{head}}:\RepSpace\to\Prob(\VV)$, $h\mapsto p_h$, parameterizing a stochastic kernel via $k_{\mathrm{head}}(h,\cdot)=p_h$. ``Learned stochasticity'' refers to the model learning distributions $p_h$ that match the data’s conditional uncertainty; the kernel itself is not a deterministic kernel $X\to\VV$.
\end{remark}

The overall single-step generation kernel $k_{\mathrm{gen}, \theta}: \ContextSeqSpaceMeas \to \VocabSpaceMeas$ is the composition $\KernelLMHead \comp \KernelBB \comp \KernelEmb$ in the category $\Stoch$. This composition precisely represents the model's learned conditional probability map $P_\Params(\cdot | \mathbf{w}_{<t})$. It is crucial to note that this formalism applies to \textbf{any} AR model, including Transformers. The attention mechanism provides a powerful, history-dependent parameterization of this single-step kernel. The subsequent sections use this representation to analyze the model's behavior.

\section{Information-Theoretic Analysis via Categorical Metrics}
\label{sec:markov-categorical-metrics}

The MC framework allows us to define principled metrics for internal analysis and to formally reason about information flow. We focus on two key applications that are central to the paper's main arguments: quantifying the information surplus exploited by speculative decoding and measuring the intrinsic stochasticity of the prediction head.

We operate within the probabilistic setting induced by a distribution $P_{\text{ctx}}$ over input contexts, corresponding to an initial state $p_{W_{<t}}: \unit \to \ContextSeqSpaceMeas$. Processing this state through the composed kernels induces distributions over the hidden state $H_t$ (state $p_{H_t}: \unit \to \RepSpaceMeas$) and the next token $W_t$ (state $p_{W_t}: \unit \to \VocabSpaceMeas$).

\subsection{Information Flow Bounds and Rationale for Speculative Decoding}
\label{ssec:speculative_decoding}

\paragraph{Setting.}
Transformers condition on the entire prefix, so the effective source is $\infty$-order. To avoid finite-order assumptions, we quantify multi-token predictability using conditional mutual-information (MI) tails.

Let $W_{t:t+K-1}=(W_t,\ldots,W_{t+K-1})$ and let the hidden summary be $H_t=\phi(W_{<t})$. The chain rule gives
\begin{equation}
\Info(H_t;W_{t:t+K-1})
\;=\; \Info(H_t;W_t)\;+\;\Info\!\big(H_t;W_{t+1:t+K-1}\mid W_t\big).
\label{eq:chain_rule_surplus}
\end{equation}
We call the second term the \emph{information surplus}
\begin{equation}
\mathsf{Surplus}_K \;\defeq\; \Info\!\big(H_t;W_{t+1:t+K-1}\mid W_t\big),
\label{eq:info_surplus_conditional}
\end{equation}
which satisfies $0\le \mathsf{Surplus}_K\le H(W_{t+1:t+K-1}\!\mid W_t)\le (K-1)\log|\VV|$. It is convenient to further decompose
\[
a_k \;\defeq\; \Info\!\big(H_t; W_{t+k}\mid W_{t:t+k-1}\big)\ \ \ge 0,\qquad k\ge1,
\]
so that $\mathsf{Surplus}_K$ is a tail sum of nonnegative per-step contributions.

\begin{proposition}[Tail sum and decay]\label{prop:surplus_decay}
For any AR LM with $H_t=\phi(W_{<t})$,
\[
\mathsf{Surplus}_K \;=\; \sum_{k=1}^{K-1} a_k, \qquad a_k\ge 0.
\]
Hence $K\mapsto \mathsf{Surplus}_K$ is nondecreasing and $\mathsf{Surplus}_K\le (K-1)\log|\VV|$. If there exists a nonincreasing envelope $\psi$ with $a_k\le \psi(k)$ and $\sum_{k\ge1}\psi(k)<\infty$, then $\mathsf{Surplus}_K\!\to\!\mathsf{Surplus}_\infty$ and
\[
0\le \mathsf{Surplus}_\infty-\mathsf{Surplus}_K \;\le\; \sum_{k\ge K}\psi(k).
\]
In particular, if $\psi(k)\le C\rho^k$ for some $C<\infty$ and $\rho\in(0,1)$, then
$\mathsf{Surplus}_\infty-\mathsf{Surplus}_K \le \frac{C\,\rho^{K}}{1-\rho}$.
\end{proposition}

\begin{corollary}[Finite memory as a special case]\label{cor:surplus_finite_memory}
If the source is $m$-th order Markov and $H_t=\phi(W_{t-m:t-1})$, then $a_k=0$ for all $k\ge m$. Thus no new surplus is added after offset $m-1$, and $\mathsf{Surplus}_K$ is constant for all $K\ge m$.
\end{corollary}

Natural language exhibits long-tailed dependence, but the tail often decays. A concise control metric is the $\varepsilon$-effective surplus length
\[
K_\varepsilon \;\defeq\; \min\{K\ge1:\ \mathsf{Surplus}_\infty-\mathsf{Surplus}_K \le \varepsilon\},
\]
which sets a principled draft length for speculative decoding under a tolerated missed-information budget $\varepsilon$. Empirically, once $\mathsf{Surplus}_K$ flattens, longer drafts yield diminishing returns unless the architecture changes.

\noindent\textbf{Two-path view of speculative decoding.}
Drafting and verification are parallel kernels on the same hidden state:
\begin{align*}
\xymatrix{
\ContextSeqSpace \hspace{3ex} \ar[r]^-{k_{\mathrm{enc}} \defeq \KernelBB \comp \KernelEmb} & \hspace{3ex} \RepSpace \ar[r]^-{k_{\mathrm{verify}}} \ar[d]_-{k_{\mathrm{draft}}} & \VV \\
& \VV^K &
}
\end{align*}
Here $k_{\mathrm{verify}}\equiv \KernelLMHead$ outputs the next-token distribution, while $k_{\mathrm{draft}}$ proposes $K$ tokens in parallel. Because these kernels have different codomains, their raw categorical entropies are not directly comparable. A meaningful diagnostic should normalize the draft entropy per token or condition sequentially, and should be interpreted together with the verifier's acceptance probability. Effective drafting requires both nonzero surplus $\mathsf{Surplus}_K$ and a draft kernel that converts this surplus into high-probability proposals.

Beyond DPI, Equation~\eqref{eq:chain_rule_surplus} gives the conditional MI budget $\mathsf{Surplus}_K$ available to parallel proposals. This quantity bounds the information available to longer drafts and helps explain where returns can saturate. Realized speedup still depends on the draft model, verifier, and acceptance rule.

\subsection{Metric: LM Head Categorical Entropy (Prediction Stochasticity)}

We quantify the intrinsic stochasticity or uncertainty associated with the final prediction step, embodied by the LM head kernel $\KernelLMHead: \RepSpaceMeas \to \VocabSpaceMeas$.
This metric is crucial for understanding how NLL training forces the model to learn the data's inherent randomness (\Cref{sec:pretraining_compression_entropy}). For a given input $h \in \RepSpace$, the categorical entropy quantifies the stochasticity of the output distribution $\KernelLMHead(h, \cdot)$. We obtain a single summary statistic by averaging this value over the distribution of hidden states $p_{H_t}$.

\begin{definition}[Categorical Entropy of $\KernelLMHead$]
Using \cref{eq:cat_entropy_def_pointwise} with $X = \RepSpace$, $Y = \VV$, and $k = \KernelLMHead$, the pointwise categorical entropy is
\begin{equation}
\CatEnt_D(\KernelLMHead)(h) \defeq D_{\VV \tens \VV} \!\left( \copyop_\VV \comp \KernelLMHead \;\middle\| \; (\KernelLMHead \tens \KernelLMHead) \comp \copyop_\RepSpace \right)(h).
\label{eq:entropy_khead_revised}
\end{equation}
\end{definition}

The categorical entropy $\CatEnt_D(\KernelLMHead)(h)$ measures the divergence between generating a correlated pair $(W, W)$ versus an independent pair $(W_1, W_2)$, where $W, W_1, W_2 \sim \KernelLMHead(h, \cdot)$. This quantifies how far the output distribution for a given $h$ is from a deterministic point mass. To obtain a single metric for the LM head's overall stochasticity, we compute its expectation with respect to the hidden state distribution $p_{H_t}$:
\begin{equation}
\AvgCatEnt_D(\KernelLMHead; p_{H_t}) \defeq \Expect_{h \sim p_{H_t}} \left[ D_{\VV \tens \VV}\left(\sum_{w \in \VV} \KernelLMHead(h, \{w\}) \deltaDirac{(w,w)} \quad \| \quad \KernelLMHead(h, \cdot) \tens \KernelLMHead(h, \cdot) \right) \right].
\label{eq:avg_cat_entropy}
\end{equation}

\noindent\textbf{Interpretation.} This metric measures the intrinsic conditional stochasticity of the LM head mapping. If $\KernelLMHead$ were deterministic (i.e., for each $h$, it mapped to a single specific $w_h$, so $p_h = \deltaDirac{w_h}$), then both measures inside the divergence would be $\deltaDirac{(w_h, w_h)}$, and the entropy would be $D(\deltaDirac{(w_h,w_h)} \| \deltaDirac{(w_h,w_h)}) = 0$. A higher value of $\AvgCatEnt_D(\KernelLMHead; p_{H_t})$ indicates greater average uncertainty or ``spread'' in the output distribution $p_h = \KernelLMHead(h, \cdot)$, meaning the kernel is inherently more stochastic. It quantifies how far the prediction process is from a deterministic assignment, measured in the geometry of $\VV \tens \VV$ induced by $D$.  In the special case $D=\KLDiv$ and finite $\VV$, $\AvgCatEnt_{\KLDiv}(\KernelLMHead;p_{H_t})\le \log|\VV|$ is automatically bounded; this bound underwrites the continuity argument used later.

For the specific case $D=\KLDiv$ and finite $\VV$, the categorical entropy equals the Shannon (conditional) entropy:
\begin{proposition}[KL case reduces to Shannon]
\label{prop:cat_equals_shannon}
For the LM head kernel $k(h,\cdot)=p_h$ over a finite vocabulary,
\[
\CatEnt_{\KLDiv}(\KernelLMHead)(h)=\ShannonEntropy(p_h),
\qquad
\AvgCatEnt_{\KLDiv}(\KernelLMHead; p_{H_t})
\,=\,\ShannonEntropy(W_t\mid H_t).
\]
\end{proposition}
\begin{proof}
Since $\sum_{w} p_h(w)\delta_{(w,w)}$ is supported on the diagonal, a direct calculation gives \\
$D_{\VV\tens\VV}\!\left(\sum_w p_h(w)\delta_{(w,w)} \;\middle\|\; p_h\!\otimes p_h\right)
=-\sum_w p_h(w)\log p_h(w)$, and averaging over $H_t$ yields $H(W_t\mid H_t)$.
\end{proof}

\section{Pretraining Objective, Compression, and Learning Intrinsic Stochasticity}
\label{sec:pretraining_compression_entropy}

A central question surrounding large language models is how the seemingly simple Autoregressive objective of next-token prediction, trained via minimizing cross-entropy loss (equivalently, negative log-likelihood or NLL), sculpts representations. The framework of Markov Categories and categorical entropy provides a lens through which to interpret this phenomenon, connecting it to compression and to matching the inherent conditional uncertainty of the data generating process.

Let $k_{\text{data}}: \ContextSeqSpaceMeas \to \VocabSpaceMeas$ be the (potentially unknown) Markov kernel representing the true data-generating process, such that $k_{\text{data}}(\mathbf{w}_{<t}, \cdot)$ corresponds to the true conditional probability measure $P_{\text{data}}(\cdot | \mathbf{w}_{<t})$ on the vocabulary $\VV$. Let $p_{W_{<t}}$ denote the marginal probability measure on the context space $\ContextSeqSpaceMeas$, derived from the underlying joint distribution $P_{\text{data}}$ over sequences observed in the training corpus.

The standard pretraining objective for an AR LM parameterized by $\theta$ is to minimize the negative log-likelihood (NLL) of the next token $w_t$ given the preceding context $\mathbf{w}_{<t}$, averaged over the training data distribution $P_{\text{data}}$. This is equivalent to minimizing the average KL divergence between the data kernel and the model kernel (\Cref{thm:nll_kl_equivalence}).
\begin{equation}
\mathcal{L}_{\mathrm{KL}}(\theta) = \Expect_{\mathbf{w}_{<t} \sim p_{W_{<t}}} [ \KLDiv( k_{\text{data}}(\mathbf{w}_{<t}, \cdot) \,\|\, k_{\mathrm{gen}, \theta}(\mathbf{w}_{<t}, \cdot) ) ]
\end{equation}
where $P_\Params(\cdot | \mathbf{w}_{<t}) = k_{\mathrm{gen}, \theta}(\mathbf{w}_{<t}, \cdot)$ and the expectation is taken over contexts $\mathbf{w}_{<t}$ drawn according to the data's marginal context distribution $p_{W_{<t}}$.

This theorem frames NLL training as driving the model kernel $k_{\mathrm{gen}, \theta}$ to match the data kernel $k_{\text{data}}$.
The connection to compression arises from Shannon's source coding theorem. The minimal average code length required to losslessly encode the next token $w_t$, given the context $\mathbf{w}_{<t}$ and using an optimal code based on the true distribution $P_{\text{data}}(\cdot | \mathbf{w}_{<t})$, is the conditional Shannon entropy $H(W_t | W_{<t})_{\text{data}}$. The cross-entropy loss $L_{\mathrm{CE}}(\Params)$ achieved by the model represents the average code length when using a code based on the model's distribution $P_\Params(\cdot | \mathbf{w}_{<t})$. Therefore, minimizing NLL (\cref{eq:kl_objective_kernels_revisited}) is equivalent to finding a model that provides the most efficient compression of the training data sequences, achieving an average code length that approaches the theoretical minimum $H(W_t | W_{<t})_{\text{data}}$. The widely discussed hypothesis that ``compression implies understanding'' posits that achieving high compression rates on complex data like natural language necessitates learning the underlying structure, rules, and statistical regularities, which may manifest as emergent capabilities.

Beyond matching the predictive distributions point-wise on average, successful NLL training implies that the model also learns to replicate the intrinsic stochasticity or uncertainty inherent in the data generation process at the prediction step. Within our framework, this intrinsic conditional stochasticity can be quantified using the concept of average categorical entropy (\cref{eq:avg_cat_entropy}). Recall the definition in Equation~\eqref{eq:avg_cat_entropy}. This quantity quantifies the average ``spread'' or non-determinism of the kernel $k$ under the input distribution $p_X$.

Let $k_{\text{head}, \theta}: \RepSpace \to \VocabSpaceMeas$ be the LM head kernel corresponding to parameters $\theta$. Let $p_{H_t, \theta}$ be the distribution over hidden states $h_t \in \RepSpace$ induced by processing contexts $\mathbf{w}_{<t} \sim p_{W_{<t}}$ through the model's encoder $\KernelBB \comp \KernelEmb$ (parameterized by $\theta$).

NLL training aligns the model's conditional uncertainty with that of the data via two ingredients: (i) KL control of per-context discrepancies (hence Hellinger control on finite vocabularies), and (ii) uniform continuity of the head-entropy functional on the compact simplex.

\begin{lemma}[Continuity of the average head-entropy functional]
\label{lem:cat_entropy_continuity}
Let $\VV$ be finite and define
\[
\Psi_D(p)\defeq D_{\VV\tens\VV}\!\big(\sum_w p(w)\delta_{(w,w)} \,\|\, p\tens p\big),
\qquad p\in\Delta^{|\VV|-1}.
\]
Assume $\Psi_D$ is continuous (hence bounded and uniformly continuous on $\Delta^{|\VV|-1}$). Let $(P^{(n)})_{n\ge1}$ and $P^\star$ be random variables in $\Delta^{|\VV|-1}$. If
\[
\mathbb E\!\big[d_H(P^{(n)},P^\star)^2\big]\to 0,
\]
then
\[
\mathbb E\big[\Psi_D(P^{(n)})\big]\to \mathbb E\big[\Psi_D(P^\star)\big].
\]
For $D=\KLDiv$, $\Psi_{\KLDiv}(p)=\ShannonEntropy(p)$ and $0\le \Psi_{\KLDiv}(p)\le \log|\VV|$.
\end{lemma}

\begin{theorem}[Convergence of average categorical entropy under NLL minimization]
\label{thm:entropy_convergence}
Let $X\sim p_{W_{<t}}$ be a random context and write
\[
p_X(\cdot)\defeq k_{\text{data}}(X,\cdot),\qquad q_{X,\theta}(\cdot)\defeq k_{\mathrm{gen},\theta}(X,\cdot).
\]
Assume realizability: there exists $\theta^\star$ such that $q_{X,\theta^\star}=p_X$ almost surely (equivalently, $\mathcal L_{\mathrm{KL}}(\theta^\star)=0$). Let $(\theta_n)$ satisfy $\mathcal L_{\mathrm{KL}}(\theta_n)\to 0$. If $\Psi_D$ from Lemma~\ref{lem:cat_entropy_continuity} is continuous, then
\[
\lim_{n\to\infty}\AvgCatEnt_D(k_{\text{head},\theta_n};\,p_{H_t,\theta_n})
= \Expect_X\big[\Psi_D(p_X)\big].
\]
In particular, for $D=\KLDiv$,
\[
\lim_{n\to\infty}\AvgCatEnt_{\KLDiv}(k_{\text{head},\theta_n};\,p_{H_t,\theta_n})
=\Expect_X[\ShannonEntropy(p_X)]
=H(W_t\mid W_{<t})_{\text{data}}.
\]
Since $H=f_{\mathrm{enc},\theta^\star}(W_{<t})$ is a deterministic function of $W_{<t}$, always $H(W_t\mid H)\ge H(W_t\mid W_{<t})$. At any realizable optimum $\theta^\star$, the data kernel factors through $H$, hence $H$ is predictively sufficient and $H(W_t\mid H)=H(W_t\mid W_{<t})_{\text{data}}$.

\end{theorem}

\Cref{thm:entropy_convergence} provides a formal basis for the claim that NLL training \emph{lets the model learn} not just the most likely next token, but also the degree of uncertainty or stochasticity associated with that prediction, as dictated by the data. By minimizing the average KL divergence $\mathcal{L}_{\mathrm{KL}}(\theta)$, the model $k_{\mathrm{gen}, \theta}$ must align its output distributions $k_{\mathrm{gen}, \theta}(\mathbf{w}_{<t}, \cdot)$ with the data distributions $k_{\text{data}}(\mathbf{w}_{<t}, \cdot)$. This alignment necessarily includes matching the ``shape'' or ``spread'' of these distributions, which is precisely what is quantified by the average categorical entropy $\AvgCatEnt_D$. The parameters $\Params$ and the compositional structure $\KernelLMHead \comp \KernelBB \comp \KernelEmb$ thus become a compressed representation capturing both the predictive dependencies and the inherent conditional randomness of the language source. This suggests that learning the correct level of stochasticity is an integral part of the compression process driven by the NLL objective, contributing to the model's ability to generate realistic and diverse text sequences.

\section{Information Geometry of Representation and Prediction Spaces}
\label{sec:infogeo}

The Markov Category framework, particularly $(\Stoch, D)$ enriched with a divergence like $\KLDiv$, provides a natural bridge to Information Geometry \citep{AmariNagaoka2000, Perrone2023Geo}. This allows for a geometric analysis of the spaces involved in AR language modeling, particularly the representation space $\RepSpace$ and the space of next-token distributions $\Prob(\VV)$.

\begin{remark}[Gauge and invariances]\label{rem:gauge}
The pullback $g^*=g_{\text{head}}^*g^{\text{FR}}$ is invariant to adding a constant to logits and, in whitened coordinates with $C_{hh}=I$, to \emph{orthogonal} reparameterizations $h=Q\tilde h$ (then $g^*(h)=Q^\top g^*(\tilde h)Q$). General $GL(d)$ changes act by congruence and thus distort eigenvalues. When comparing spectra across checkpoints or layers, it is therefore natural to work in whitened coordinates or report coordinate-free summaries such as eigenvalue ratios and principal subspace overlaps.
\end{remark}

The space $\Prob(\VV)$ of probability distributions over the finite vocabulary $\VV$ forms a $(|\VV|-1)$-dimensional simplex $\Delta^{|\VV|-1}$. This space possesses a well-defined Riemannian geometry induced by the Fisher-Rao information metric $g^{\text{FR}}$, whose components in a local coordinate system $\xi = (\xi_1, \dots, \xi_{|\VV|-1})$ for a distribution $p_\xi \in \Prob(\VV)$ are given by:
\begin{equation}
g^{\text{FR}}_{ij}(\xi) = \sum_{w \in \VV} p_\xi(w) \frac{\partial \log p_\xi(w)}{\partial \xi_i} \frac{\partial \log p_\xi(w)}{\partial \xi_j} = \Expect_{W \sim p_\xi} \left[ \frac{\partial \log p_\xi(W)}{\partial \xi_i} \frac{\partial \log p_\xi(W)}{\partial \xi_j} \right].
\label{eq:fisher_rao}
\end{equation}
This metric quantifies the local distinguishability between nearby probability distributions, measuring the distance in terms of expected squared log-likelihood ratio gradients. The geometry of $\Prob(\VV)$ also includes dual affine connections ($\pm \alpha$-connections) related to the KL divergence, providing a richer dually flat structure \citep{AmariNagaoka2000}.

The LM Head kernel $\KernelLMHead: \RepSpaceMeas \to \VocabSpaceMeas$ corresponds to a deterministic mapping from a hidden state $h \in \RepSpace \cong \mathbb{R}^{d_{\mathrm{model}}}$ to a probability distribution $p_h \defeq \KernelLMHead(h, \cdot) \in \Prob(\VV)$. Let $g_{\text{head}}: \RepSpace \to \Prob(\VV)$ denote this mapping, $p_h = g_{\text{head}}(h)$. Typically, this involves a linear layer followed by softmax: $g_{\text{head}}(h) = \softmax(W h)$ where $W \in \mathbb{R}^{|\VV| \times d_{\mathrm{model}}}$. This mapping $g_{\text{head}}$ allows us to pull back the geometric structure from $\Prob(\VV)$ onto the representation space $\RepSpace$.

To avoid collision with the token random variables $W_t$, we denote the output weight matrix by $W_{\mathrm{out}}\in\mathbb{R}^{|\VV|\times d_{\mathrm{model}}}$ below, so $g_{\text{head}}(h)=\softmax(W_{\mathrm{out}}h+b)$.

Specifically, the Fisher-Rao metric $g^{\text{FR}}$ on $\Prob(\VV)$ induces a positive-semidefinite pullback tensor $g^* = g_{\text{head}}^* g^{\text{FR}}$ on $\RepSpace$. It is a genuine Riemannian metric only on directions where the head map has full local rank; otherwise it is a degenerate metric tensor. At a point $h \in \RepSpace$, its components are given by:
\begin{equation}
g^*_{ab}(h) = \sum_{i,j} g^{\text{FR}}_{ij}(g_{\text{head}}(h)) \frac{\partial (g_{\text{head}}(h))_i}{\partial h_a} \frac{\partial (g_{\text{head}}(h))_j}{\partial h_b}, \quad a, b \in \{1, \dots, d_{\mathrm{model}}\},
\label{eq:pullback_metric}
\end{equation}
where $h_a, h_b$ are coordinates of $h \in \RepSpace$, and $(g_{\text{head}}(h))_i, (g_{\text{head}}(h))_j$ represent local coordinates of the output distribution $p_h \in \Prob(\VV)$ (e.g., probabilities of specific tokens, possibly excluding one due to the sum-to-one constraint). The term $\frac{\partial (g_{\text{head}}(h))_i}{\partial h_a}$ is the Jacobian of the LM head map $g_{\text{head}}$ evaluated at $h$.

Let $J(h)$ denote this Jacobian matrix ($|\VV|-1 \times d_{\mathrm{model}}$ or $|\VV| \times d_{\mathrm{model}}$ depending on coordinates). Then $g^*(h) = J(h)^{\top} g^{\text{FR}}(g_{\text{head}}(h)) J(h)$.
The significance of this pullback metric $g^*$ lies in its connection to the local distinguishability of output distributions under perturbations of the input hidden state, as measured by divergences like KL divergence.

\begin{theorem}[Pullback Metric and Local Divergence]
Assume $p_h\in\mathrm{int}\,\Delta^{|\VV|-1}$ (i.e., $p_h(w)>0$ for all $w$; see Assumption~\ref{assump:standing}).
Let $g_{\text{head}}: \RepSpace \to \Prob(\VV)$ be the smooth map corresponding to the LM head kernel. Let $h \in \RepSpace$ and $v \in T_h\RepSpace \cong \RepSpace$. Consider the distributions $p_h = g_{\text{head}}(h)$ and $p_{h+\epsilon v} = g_{\text{head}}(h+\epsilon v)$ for small $\epsilon$. The KL divergence between these output distributions, for small $\epsilon$, is locally approximated by the quadratic form defined by the pullback metric $g^*(h)$:
\begin{equation}
\KLDiv(p_{h+\epsilon v} \,\|\, p_h) = \frac{1}{2} \epsilon^2 g^*(h)(v, v) + O(\epsilon^3)
\label{eq:gstar_kl_relation}
\end{equation}
where $g^*(h)(v, v) = \sum_{a,b=1}^{d_{\mathrm{model}}} g^*_{ab}(h) v_a v_b$. A similar relationship holds for symmetric KL divergence. More generally, for a sufficiently smooth $f$-divergence, the leading term is $\frac{f''(1)}{2}\epsilon^2 g^*(h)(v,v)$.
\end{theorem}
\begin{proof}
Let $\xi$ be a local coordinate system for $\Prob(\VV)$ around $p_h$. The KL divergence between two nearby distributions $p_\xi$ and $p_{\xi'}$ can be expanded around $p_\xi$ as \citep{AmariNagaoka2000}:
\[ \KLDiv(p_{\xi'} \,\|\, p_\xi) = \frac{1}{2} \sum_{i,j} g^{\text{FR}}_{ij}(\xi) (\xi'_i - \xi_i) (\xi'_j - \xi_j) + O(\|\xi' - \xi\|^3). \]
Let $\xi(h)$ denote the coordinates of $p_h = g_{\text{head}}(h)$. For $p_{h+\epsilon v}$, the coordinates are $\xi(h+\epsilon v)$. By Taylor expansion in $\epsilon$:
\[ \xi_i(h+\epsilon v) = \xi_i(h) + \epsilon \sum_{a=1}^{d_{\mathrm{model}}} \frac{\partial \xi_i}{\partial h_a}(h) v_a + O(\epsilon^2). \]
Thus, $\xi_i(h+\epsilon v) - \xi_i(h) = \epsilon J_{ia}(h) v_a + O(\epsilon^2)$, where $J(h)$ is the Jacobian matrix of the map $h \mapsto \xi(h)$ (i.e., the Jacobian of $g_{\text{head}}$ in local coordinates $\xi$). Substituting this into the KL expansion:
\begin{align*}
\KLDiv(p_{h+\epsilon v} \,\|\, p_h) &= \frac{1}{2} \sum_{i,j} g^{\text{FR}}_{ij}(\xi(h)) \left( \epsilon \sum_a J_{ia}(h) v_a \right) \left( \epsilon \sum_b J_{jb}(h) v_b \right) + O(\epsilon^3) \\
&= \frac{1}{2} \epsilon^2 \sum_{a,b} \left( \sum_{i,j} J_{ia}(h) g^{\text{FR}}_{ij}(\xi(h)) J_{jb}(h) \right) v_a v_b + O(\epsilon^3) \\
&= \frac{1}{2} \epsilon^2 \sum_{a,b} (J(h)^{\top} g^{\text{FR}}(\xi(h)) J(h))_{ab} v_a v_b + O(\epsilon^3).
\end{align*}
The term $J(h)^{\top} g^{\text{FR}}(\xi(h)) J(h)$ is precisely the matrix representation of the pullback metric $g^*(h)$ in the standard coordinates of $\RepSpace \cong \mathbb{R}^{d_{\mathrm{model}}}$, derived from \cref{eq:pullback_metric}. Thus, $\KLDiv(p_{h+\epsilon v} \,\|\, p_h) = \frac{1}{2} \epsilon^2 g^*(h)(v, v) + O(\epsilon^3)$. The result for other well-behaved $f$-divergences follows from their similar second-order expansion involving $g^{\text{FR}}$.
\end{proof}

This theorem formally establishes that the pullback metric $g^*$ measures how sensitive the output distribution $p_h$ is to infinitesimal changes in the hidden state $h$, where sensitivity is gauged by the local divergence (specifically, KL divergence, relating to the Fisher-Rao metric) in the output space $\Prob(\VV)$.

\begin{lemma}[Softmax--linear head: closed form of $g^*$]\label{lem:softmax_closed_form}
Suppose $p_h=\softmax(W_{\mathrm{out}}h+b)$ with $W_{\mathrm{out}}\in\mathbb R^{|\VV|\times d}$. Then
\[
g^*(h)\;=\;W_{\mathrm{out}}^\top\!\left[\mathrm{Diag}(p_h)-p_h p_h^\top\right]\!W_{\mathrm{out}}.
\]
\end{lemma}
\begin{proof}
For softmax logits $z_w=\langle W_{\mathrm{out},w},h\rangle+b_w$,
$\nabla_h\log p_h(w)=W_{\mathrm{out}}^\top(e_w-p_h)$.
Therefore $g^*(h)=\mathbb E_{W\sim p_h}[\nabla_h\log p_h(W)\nabla_h\log p_h(W)^\top]=W_{\mathrm{out}}^\top(\mathrm{Diag}(p_h)-p_hp_h^\top)W_{\mathrm{out}}$.
\end{proof}

\begin{lemma}[FR–Lipschitzness and Hellinger control]\label{lem:lipschitz_head}
Let $p_h=\softmax(W_{\mathrm{out}}h+b)$ with $W_{\mathrm{out}}\in\mathbb R^{|\VV|\times d}$. Then for all $h$,
\[
g^*(h)
=W_{\mathrm{out}}^\top\!\big(\mathrm{Diag}(p_h)-p_h p_h^\top\big)W_{\mathrm{out}}
\ \preceq\ \frac{\|W_{\mathrm{out}}\|_2^2}{2}\,I_d.
\]
Consequently, the head map $g_{\mathrm{head}}:h\mapsto p_h$ is globally Lipschitz from Euclidean to Fisher--Rao:
\[
d_{\mathrm{FR}}\!\big(p_h,p_{h'}\big)\ \le\ \frac{\|W_{\mathrm{out}}\|_2}{\sqrt{2}}\;\|h-h'\|_2,\qquad \forall\,h,h'.
\]
Moreover, using $d_H(p,q)\le d_{\mathrm{FR}}(p,q)/(2\sqrt{2})$ on the simplex,
\[
d_H\!\big(p_h,p_{h'}\big)\ \le\ \frac{\|W_{\mathrm{out}}\|_2}{4}\,\|h-h'\|_2,\qquad \forall\,h,h'.
\]
\end{lemma}

\begin{remark}[FR–Hellinger relation]
On the probability simplex, $d_{\mathrm{FR}}(p,q)=2\arccos\!\sum_w\sqrt{p(w)q(w)}$ and
$d_H(p,q)=\sqrt{1-\sum_w\sqrt{p(w)q(w)}}=\sqrt{2}\,\sin\!\big(d_{\mathrm{FR}}(p,q)/4\big)$.
Thus $d_H(p,q)\le d_{\mathrm{FR}}(p,q)/(2\sqrt{2})$ by $\sin x\le x$, and also $d_{\mathrm{FR}}(p,q)\le \pi\sqrt{2}\,d_H(p,q)$ by $\sin x\ge 2x/\pi$ on $x\in[0,\pi/2]$.
\end{remark}
\begin{proof}[Proof sketch]
By Lemma~\ref{lem:softmax_closed_form}, $g^*(h)=W_{\mathrm{out}}^\top(\mathrm{Diag}(p_h)-p_hp_h^\top)W_{\mathrm{out}}$. The covariance factor $\mathrm{Diag}(p_h)-p_hp_h^\top$ has spectral norm at most $1/2$, hence $g^*(h)\preceq \frac{\|W_{\mathrm{out}}\|_2^2}{2}I_d$. For any $h,h'$, consider the straight path $h_t=(1-t)h+t h'$. The Fisher--Rao distance is upper bounded by the length of the image path: $d_{\mathrm{FR}}(p_h,p_{h'})\le \int_0^1 \sqrt{(h'-h)^\top g^*(h_t)(h'-h)}\,\dd t \le \frac{\|W_{\mathrm{out}}\|_2}{\sqrt2}\|h-h'\|_2$. The Hellinger bound follows from $d_H\le d_{\mathrm{FR}}/(2\sqrt2)$.
\end{proof}

\begin{remark}[Weight tying]
If the head ties weights with the input embedding ($W_{\mathrm{out}}=E^\top$), bounds and spectra involving $W_{\mathrm{out}}$ inherit the input-embedding geometry. In particular, in whitened coordinates (\Cref{cor:whitened_case}), the principal directions of $g^*$ align with the principal directions of $E$’s covariance weighted by $p_h$.
\end{remark}

\begin{remark}[Local vs. global distances]
For small displacements, the FR geodesic distance agrees to second order with the quadratic approximation in \Cref{eq:gstar_kl_relation}. For larger moves, one must integrate along a curve; the straight-line path in $\RepSpace$ yields an upper bound on the FR distance between $g_{\mathrm{head}}(h)$ and $g_{\mathrm{head}}(h')$.
At \emph{interior} points of the simplex, FR geodesic distance and (symmetric) KL are locally equivalent (second order). Near the boundary, restrict attention to a high-probability subset where $p_h(w)\ge \delta>0$ so that constants in the local equivalence remain controlled; equivalently, operate with tempered/clipped logits to maintain interiority.
\end{remark}

\begin{remark}[Pullback Metric as Expected Score Outer Product]
The Fisher-Rao metric $g^{\text{FR}}$ is the expected outer product of the score function $\nabla_\xi \log p_\xi(W)$. This property pulls back to $\RepSpace$. Let $p_h(w) = \KernelLMHead(h, \{w\})$. The score vector for token $w$ with respect to the representation is $\nabla_h \log p_h(w) \in \RepSpace$. The pullback metric tensor is precisely the expected outer product of this score:
\begin{equation}
g^*(h) = \Expect_{W \sim p_h} [(\nabla_h \log p_h(W)) (\nabla_h \log p_h(W))^{\top}].
\label{eq:gstar_score}
\end{equation}
This directly connects the information geometry of $\RepSpace$ to the sensitivity of log-probabilities to changes in the representation $h$. This score vector $\nabla_h \log p_h(W)$ is analogous to that used in score-based generative models, but here taken with respect to the conditioning variable $h$.

\noindent\textbf{Connection to optimization.} Under standard regularity, $g^*(h)$ coincides with the Fisher information (and Gauss–Newton) matrix of the per-step NLL with respect to $h$, furnishing a direct link between second-order training dynamics and the pullback geometry on $\RepSpace$.

\end{remark}

The rank of the pullback metric depends on the dimensions of the spaces involved.
\begin{proposition}[Rank of the Pullback Metric]
\label{prop:rank_pullback_metric}
The rank of the pullback Fisher-Rao metric $g^*(h)$ at a point $h \in \RepSpace$ is bounded by the minimum of the representation dimension and the dimension of the probability simplex:
\begin{equation}
\mathrm{rank}(g^*(h)) \le \min(d_{\mathrm{model}}, |\VV| - 1).
\end{equation}
\end{proposition}
\begin{proof}
The pullback metric $g^*(h)$ is defined as $g^*(h) = J(h)^{\top} g^{\text{FR}}(g_{\text{head}}(h)) J(h)$, where $J(h)$ is the Jacobian of the map $g_{\text{head}}: \RepSpace \to \Prob(\VV)$ (represented in appropriate local coordinates). The dimension of $\RepSpace$ is $d_{\mathrm{model}}$, and the dimension of $\Prob(\VV)$ is $d_{\text{prob}} = |\VV| - 1$. The Jacobian $J(h)$ is a $d_{\text{prob}} \times d_{\mathrm{model}}$ matrix. The Fisher-Rao metric $g^{\text{FR}}$ at $g_{\text{head}}(h)$ is a $d_{\text{prob}} \times d_{\text{prob}}$ positive definite matrix (and thus has rank $d_{\text{prob}}$).
Using the property that $\mathrm{rank}(A^{\top}BA) = \mathrm{rank}(A)$ if $B$ is positive definite, we have $\mathrm{rank}(g^*(h)) = \mathrm{rank}(J(h))$. The rank of a matrix is bounded by its dimensions, so
\[ \mathrm{rank}(g^*(h)) \le \min(d_{\mathrm{model}}, d_{\text{prob}}) = \min(d_{\mathrm{model}}, |\VV| - 1). \]
On the interior of the simplex (Assumption~\ref{assump:standing}), $g^{\text{FR}}$ is positive definite; near the boundary the argument applies on a high-probability subset with $p_h(w)\ge \delta>0$.
\end{proof}

\subsection{Interpretation and Implications}

This geometric perspective provides several insights:
\begin{itemize}
\item The quadratic form $g^*(h)(v, v)$ quantifies the local distinguishability (via KL divergence, \cref{eq:gstar_kl_relation}) between the output distributions $p_h$ and $p_{h+\epsilon v}$. It measures how sensitive the model's prediction is to perturbations of the hidden state $h$ in a direction $v$. Directions $v$ with large $g^*(h)(v, v)$ correspond to changes in $h$ that significantly alter the output distribution.
\item In modern LMs, the representation dimension is usually much smaller than the vocabulary size ($d_{\mathrm{model}}\ll |\VV|$). By Prop.~\ref{prop:rank_pullback_metric}, $\mathrm{rank}(g^*(h))\le d_{\mathrm{model}}$. This upper bound does \emph{not} imply degeneracy on $\RepSpace$: if the head Jacobian has full column rank modulo the all-ones logit direction, $g^*(h)$ can be full rank on $\mathbb R^{d_{\mathrm{model}}}$. The important phenomenon is instead \textit{anisotropy}: some directions have much larger predictive sensitivity than others. \label{rem:typical_rank_degeneracy}
\item The eigenvalues and eigenvectors of the matrix for $g^*(h)$ reveal the local principal directions of predictive sensitivity in $\RepSpace$. Directions with large eigenvalues are those where small changes in $h$ induce large changes, geometrically measured by $g^{\text{FR}}$, in the predicted distribution $p_h$. These directions are important for the head's current prediction map; whether they coincide with globally important semantic directions depends on the encoder distribution and the region of $\RepSpace$ being analyzed.
\end{itemize}

Thus separation should be understood through the head map. Contexts with different predictive futures need not be far in every Euclidean direction of $\RepSpace$; they must differ along directions visible to $g_{\text{head}}$. The geometry induced by $g^*$ characterizes this local separation capability. Training shapes the encoder $(\KernelBB\comp\KernelEmb)$ and the LM head $\KernelLMHead$ so that contexts with different true conditionals are mapped to hidden states whose images under $g_{\text{head}}$ are appropriately separated in $\Prob(\VV)$.

\section{NLL as Implicit Spectral Contrastive Learning}
\label{sec:repr_learning_theory}

A central thesis of this work is that the simple objective of minimizing the negative log-likelihood (NLL) of the next token (\cref{eq:kl_objective_kernels_revisited}) implicitly functions as a powerful form of contrastive learning. While lacking the explicit positive/negative pairs of standard contrastive methods, we prove that NLL optimization inherently structures the learned representation space $\RepSpace$ according to predictive similarity. It achieves this by implicitly solving a spectral objective that aligns the geometry of representations with the underlying predictive structure of the data $P_{\text{data}}(\cdot|x)$, a principle we formalize by connecting NLL to the eigenspectrum of a predictive similarity operator \citep{haochen2021provable, Tan2023ContrastiveSpectral}.

Let $f_{\text{enc}}: \ContextSeqSpace \to \RepSpace$ denote the deterministic encoder mapping a context sequence $x = \mathbf{w}_{<t}$ to its hidden representation $h_x = f_{\text{enc}}(x)$, implemented by the composition $\KernelBB \comp \KernelEmb$. Let $g_{\text{head}}: \RepSpace \to \Prob(\VV)$ be the deterministic mapping from the hidden state to the next-token distribution, corresponding to the LM head kernel $\KernelLMHead$, such that $p_\theta(\cdot | x) = g_{\text{head}}(h_x)$. The training objective is to minimize the expected KL divergence over the context distribution $\mu_{ctx} = p_{W_{<t}}$:

\begin{equation}
\mathcal{L}(\theta) = \Expect_{x \sim \mu_{ctx}} [ \KLDiv( P_{\text{data}}(\cdot | x) \,\|\, g_{\text{head}}(f_{\text{enc}}(x)) ) ]
\label{eq:nll_objective_func_revisited}
\end{equation}
where $P_{\text{data}}(\cdot | x)$ represents the true conditional distribution of the next token given context $x$, assumed to be derived from the data-generating process.

Successful optimization of $\mathcal{L}(\theta)$ drives the model's output distribution $p_\theta(\cdot|x) = g_{\text{head}}(h_x)$ towards the target distribution $P_{\text{data}}(\cdot|x)$ in the sense of minimizing average KL divergence. As we argue below, this fundamental requirement indirectly imposes geometric constraints on the distribution of representations $h_x = f_{\text{enc}}(x)$ in $\RepSpace$.

\subsection{Constraint on Output Distribution Approximation}

Minimizing the NLL loss (\cref{eq:nll_objective_func_revisited}) directly forces the model's predicted distribution $p_\theta(\cdot|x)$ to closely approximate the target distribution $P_{\text{data}}(\cdot|x)$. This closeness can be measured not only by KL divergence but also by other standard metrics on probability distributions, due to well-known inequalities relating them.

\begin{theorem}[Output Distribution Approximation Constraint]
\label{thm:implicit_contrastive_constraint}
Assume the model parameters $\theta$ yield a small average KL divergence $\mathcal{L}_{\mathrm{KL}}(\theta)\defeq \Expect_{x \sim \mu_{ctx}}[\KLDiv(P_{\text{data}}(\cdot|x) \| p_\theta(\cdot|x))]$, where $p_\theta(\cdot|x) = g_{\text{head}}(f_{\text{enc}}(x))$. With our global convention for Hellinger, $d_H^2(p,q)\le \frac{1}{2}\KLDiv(p\|q)$ on finite alphabets. Therefore,
\begin{equation}
\Expect_{x \sim \mu_{ctx}}[d_H(P_{\text{data}}(\cdot|x), p_\theta(\cdot|x))^2] \le \tfrac12\,\mathcal{L}_{\mathrm{KL}}(\theta).
\label{eq:output_dist_approx}
\end{equation}
Consequently, if the model fits the data well ($\mathcal{L}_{\mathrm{KL}}(\theta)$ is small), then for any pair $(x,x')$,
\begin{align}
\big|d_H(p_\theta(\cdot|x), p_\theta(\cdot|x')) - d_H(P_{\text{data}}(\cdot|x), P_{\text{data}}(\cdot|x'))\big|
\le \epsilon_x+\epsilon_{x'},\qquad \epsilon_x:=d_H(P_{\text{data}}(\cdot|x), p_\theta(\cdot|x)),
\label{eq:output_approx_implication}
\end{align}
and $\mathbb P(\epsilon_X\ge \delta)\le \tfrac{1}{2}\,\mathcal L_{\mathrm{KL}}(\theta)/\delta^{2}$ by Markov. Hence, for independent $X,X'\sim\mu_{\mathrm{ctx}}$,
\[
\mathbb P\!\left(
\big|d_H(p_\theta^X,p_\theta^{X'})-d_H(p_{\mathrm{data}}^X,p_{\mathrm{data}}^{X'})\big|>2\delta
\right)
\le
\frac{\mathcal L_{\mathrm{KL}}(\theta)}{\delta^2},
\]
where $p_\theta^X=p_\theta(\cdot|X)$ and $p_{\mathrm{data}}^X=P_{\mathrm{data}}(\cdot|X)$. Thus pairwise predictive distances are approximated on typical random pairs whenever the average KL is small.
\end{theorem}
\begin{proof}[Proof Sketch]
Apply $d_H^2\le \frac{1}{2}\,\mathrm{KL}$ pointwise with $p = P_{\text{data}}(\cdot|x)$ and $q = p_\theta(\cdot|x)$ and take expectations to obtain \Cref{eq:output_dist_approx}. The pairwise inequality follows from the triangle inequality for $d_H$. Markov's inequality gives the one-point tail bound, and a union bound over independent $X,X'$ gives the displayed random-pair bound.
\end{proof}
Different normalizations of $d_{\mathrm{TV}}$ and $d_H$ only change constants; we use the convention fixed in the preliminaries throughout.

This theorem formalizes the intuition that minimizing the NLL objective forces the model's predictions to mirror the structure of the true predictive distributions, specifically in terms of their pairwise distances.

\subsection{Consequences for Representation Geometry}

\Cref{thm:implicit_contrastive_constraint} establishes that predictively dissimilar contexts $x, x'$ must lead to distinct model output distributions $p_\theta(\cdot|x), p_\theta(\cdot|x')$. Since $p_\theta(\cdot|x) = g_{\text{head}}(h_x)$ and $p_\theta(\cdot|x') = g_{\text{head}}(h_{x'})$, this requirement imposes constraints on the corresponding representations $h_x = f_{\text{enc}}(x)$ and $h_{x'} = f_{\text{enc}}(x')$. Specifically, $h_x$ and $h_{x'}$ must differ in ways that are discernible by the head mapping $g_{\text{head}}$.
The information geometry of the head mapping, captured by the pullback metric $g^*(h)$ (\Cref{sec:infogeo}), determines which differences in representation space are discernible.

\begin{corollary}[Implicit Representation Separation]
\label{coroll:representation_separation}
Assume the model fits the data well ($\mathcal{L}(\theta)$ is small). Let $p_h\defeq g_{\mathrm{head}}(h)$ and $h_x=f_{\mathrm{enc}}(x)$. For two contexts $x,x'$, set $v\defeq h_x-h_{x'}$ and consider the straight path $h_t=(1-t)h_{x'}+t h_x$. The image path $\gamma(t)=g_{\mathrm{head}}(h_t)$ in the output simplex has length
\[
L(\gamma)=\int_0^1 \sqrt{g^*(h_t)(v,v)}\,\dd t,
\]
so the Fisher--Rao geodesic distance satisfies
\begin{equation}
d_{\mathrm{FR}}\!\big(p_{h_x},p_{h_{x'}}\big)
\ \le\ \int_0^1 \sqrt{g^*(h_t)(v,v)}\,\dd t
\ \le\ \sqrt{\int_0^1 g^*(h_t)(v,v)\,\dd t}.
\label{eq:repsep_path_bound}
\end{equation}
Thus $d_{\mathrm{FR}}(p_{h_x},p_{h_{x'}})^2 \le \int_0^1 g^*(h_t)(v,v)\,\dd t$. Combined with \Cref{thm:implicit_contrastive_constraint} and the monotone relation between Hellinger and Fisher--Rao distance on the simplex, predictively dissimilar contexts force a large right-hand side in Eq.~\eqref{eq:repsep_path_bound}. Equivalently, $v$ must have substantial components along directions where $g^*$ is large, unless the head is locally insensitive along the path.
\end{corollary}
\begin{proof}[Proof Sketch]
From \Cref{thm:implicit_contrastive_constraint}, if $d_{\text{out}}(P_{\text{data}}(\cdot|x), P_{\text{data}}(\cdot|x'))$ is large, then $d_{\text{out}}(g_{\text{head}}(h_x), g_{\text{head}}(h_{x'}))$ must also be large. The squared distance between two points in a Riemannian manifold is related to the integrated metric along a geodesic. For small distances in output space, we have $d_{\text{out}}(p, q)^2 \approx \KLDiv(p\|q)$, which from Eq.~\eqref{eq:gstar_kl_relation} is related to the pullback metric $g^*$. A large distance between $g_{\text{head}}(h_x)$ and $g_{\text{head}}(h_{x'})$ implies a large integrated path length according to the pullback geometry, forcing $h_x$ and $h_{x'}$ to differ along directions where $g^*$ is large.
\end{proof}

This corollary establishes that NLL minimization implicitly acts like a contrastive learning objective: it pushes representations $h_x, h_{x'}$ apart if their corresponding contexts are predictively dissimilar. This differential pressure based on predictive similarity forms the basis for our connection to spectral methods.

\subsection{Predictive Similarity Kernels}

The preceding analysis suggests that NLL shapes the representation geometry based on the \textit{dissimilarity} between the true next-token distributions $P_{\text{data}}(\cdot|x)$. To connect this to spectral methods, which operate on similarity structures, we formalize the complementary notion of \textit{predictive similarity}.

\begin{definition}[Predictive Similarity Kernel]
\label{def:predictive_similarity_kernel}
Let $p_x \defeq P_{\text{data}}(\cdot|x)$ denote the true conditional distribution for context $x$. A predictive similarity kernel is a bounded symmetric real-valued function $K: \ContextSeqSpace \times \ContextSeqSpace \to \mathbb{R}$ quantifying the similarity between $p_x$ and $p_{x'}$. For graph-Laplacian and Dirichlet-energy statements below, we additionally require $K(x,x')\ge0$ so that it can be interpreted as an edge weight. For operator and CCA statements, the important condition is positive semidefiniteness of the quadratic form; after centering or whitening, such kernels may take negative pairwise values. Examples include:
\begin{itemize}
\item \textbf{Bhattacharyya Coefficient Kernel}: $K_{\text{BC}}(x, x') \defeq \text{BC}(p_x, p_{x'}) = \sum_{w \in \VV} \sqrt{p_x(w) p_{x'}(w)}$. This measures the cosine similarity between the square-root vectors $(\sqrt{p_x(w)})_w$. Under our global convention $d_H^2(p,q)=1-\sum_w\sqrt{p(w)q(w)}$, we have $d_H^2(p_x,p_{x'}) = 1 - K_{\text{BC}}(x,x')$, and $K_{\text{BC}}$ is positive semidefinite. High $K_{\text{BC}}$ corresponds to low $d_H$.
\item \textbf{Hellinger-based Kernel (Gaussian Kernel on $\sqrt{p}$)}: $K_{\text{H}}(x, x') \defeq \exp(-\beta d_H^2(p_x, p_{x'}))$ for some scale $\beta > 0$. Since $d_H^2(p,q)=1-\langle \sqrt p,\sqrt q\rangle$, this kernel equals $e^{-\beta}\exp(\beta\langle \sqrt{p_x},\sqrt{p_{x'}}\rangle)$ and is positive semidefinite.
\item \textbf{Expected Likelihood Kernel (Linear Kernel)}: $K_{\text{Lin}}(x, x') \defeq \langle p_x, p_{x'} \rangle = \sum_{w \in \VV} p_x(w) p_{x'}(w)$. This is the standard linear kernel (inner product) between probability vectors $p_x, p_{x'}$ and is positive semidefinite. High values indicate significant overlap between the distributions. It can be interpreted as the expected likelihood $p_{x'}(W)$ under $W \sim p_x$.
\item \textbf{Divergence-based Kernels}: $K(x,x')=\exp(-\beta S(p_x,p_{x'}))$ for a symmetric divergence $S$, such as Jensen--Shannon divergence. Positive semidefiniteness is not automatic for arbitrary choices of $S$ and must be verified separately.
\end{itemize}
In general, larger values of $K(x, x')$ indicate higher predictive similarity. To act on representations, we disintegrate $K$ through the encoder $f_{\mathrm{enc}}$ by defining the induced kernel on $\RepSpace$:
\begin{equation}
\widetilde K(h,h') \triangleq \mathbb E\!\big[K(X,X') \;\big|\; f_{\mathrm{enc}}(X)=h, f_{\mathrm{enc}}(X')=h'\big],
\label{eq:Ktilde_def}
\end{equation}
whenever the conditional expectation exists. Since all spaces here are standard Borel, regular conditional probabilities exist; hence, the disintegration above is well-defined. This produces a bounded symmetric measurable kernel on $(\RepSpace,\mu)$, where $\mu=(f_{\mathrm{enc}})_\#\mu_{ctx}$, suitable for operator-theoretic analysis.
For the operator results below we assume $K$ is \emph{positive semidefinite} (PSD), meaning that its integral quadratic form is nonnegative for all bounded measurable test functions. The examples $K_{\text{BC}}$, $K_{\text{H}}$, and $K_{\text{Lin}}$ are PSD. For graph-energy results, we also assume pointwise nonnegativity.
\end{definition}

\begin{lemma}[PSD preserved under disintegration]\label{lem:psd_disintegration}
If $K$ is PSD on $\ContextSeqSpace$ in the quadratic-form sense, then $\widetilde K$ defined in \eqref{eq:Ktilde_def} is PSD on $(\RepSpace,\mu)$. Consequently, for all $\psi\in L^2(\RepSpace,\mu)$,
\[
\iint \psi(h)\,\widetilde K(h,h')\,\psi(h')\,\mu(\dd h)\mu(\dd h') \;\ge\; 0.
\]
\end{lemma}
\begin{proof}[Proof sketch]
Let $(X,X')\sim\mu_{ctx}\otimes\mu_{ctx}$ and $H=f_{\mathrm{enc}}(X)$, $H'=f_{\mathrm{enc}}(X')$. For any bounded measurable $\psi$, by the tower property,
\[
\mathbb E[\psi(H)\psi(H')K(X,X')] \;=\; \mathbb E\!\left[\,\mathbb E[K(X,X')\mid H,H']\,\psi(H)\psi(H')\right]
= \iint \psi(h)\widetilde K(h,h')\psi(h')\,\dd\mu(h)\dd\mu(h').
\]
Since $K$ is PSD, the left-hand side is $\ge 0$, hence so is the right-hand side.
\end{proof}

\subsection{Connection to Graph Laplacian and Dirichlet Energy Minimization}

Consider an undirected graph where contexts $x \in \ContextSeqSpace$ are nodes distributed according to $\mu_{ctx}$, and edge weights are given by a bounded nonnegative symmetric predictive similarity kernel $K(x, x')$.
The quadratic form of the associated graph Laplacian corresponds to the Dirichlet energy, which measures how ``smooth'' a function $\phi$ (e.g., a 1D projection of the representations) is over the graph.
\begin{equation}
\mathcal{E}_K(\phi) \defeq \frac{1}{2} \iint K(x, x') (\phi(x) - \phi(x'))^2 \, \mu_{ctx}(\dd x) \mu_{ctx}(\dd x') = \langle \phi, \Delta_K \phi \rangle_{L^2(\mu_{ctx})}.
\label{eq:dirichlet_energy}
\end{equation}
Spectral clustering aims to find embeddings (represented by functions $\phi$) that minimize this energy subject to constraints, effectively mapping similar contexts close together. The NLL objective, through \Cref{coroll:representation_separation}, exerts a related pressure.

Recent work \citep{park2024iclr} connects in-context learning to Dirichlet energy minimization on a task-similarity graph. Our work shows this is a foundational principle of NLL pre-training itself, where similarity is defined by the intrinsic next-token distributions $P_{\text{data}}(\cdot|x)$.

\begin{proposition}[Dirichlet energy bound under local bi-Lipschitzness]
\label{prop:nll_dirichlet}
Fix a bounded nonnegative symmetric predictive-similarity kernel $K$ and define, for $v\in\RepSpace$, the projection $\phi_v(x)=\langle h_x, v\rangle$ with $h_x=f_{\mathrm{enc}}(x)$. Assume moreover that $g_{\mathrm{head}}$ is locally bi-Lipschitz on a high-probability compact set: there exist $0<m\le L<\infty$ such that for all $h,h'$ in this set,
\[
m\,\|h-h'\|_2 \ \le\ d_{\mathrm{FR}}(p_h,p_{h'})\ \le\ L\,\|h-h'\|_2,
\]
and use $d_{\mathrm{FR}}(p,q)\le \pi\sqrt2\,d_H(p,q)$, so that
\[
\|h-h'\|_2 \ \le\ \frac{\pi\sqrt2}{m}\,d_H(p_h,p_{h'}).
\]
Suppose
\[
\varepsilon \;\defeq\; \Expect_{x\sim\mu_{ctx}}\!\Big[d_H\!\big(P_{\mathrm{data}}(\cdot\!\mid\!x),\,p_\theta(\cdot\!\mid\!x)\big)^2\Big]
\]
is small and let $\|K\|_\infty<\infty$. Then
\begin{align*}
\mathcal{E}_K(\phi_v)
\;&\le\;
\frac{\pi^2\|v\|^2}{m^{2}}\, \iint K(x,x')\, d_H\!\big(p_\theta(\cdot\!\mid\!x),p_\theta(\cdot\!\mid\!x')\big)^2 \,\mu_{ctx}(\dd x)\mu_{ctx}(\dd x') \\
\;&\le\;
C_{\mathrm{geom}}\,
\iint K(x,x')\, d_H\!\big(P_{\mathrm{data}}(\cdot\!\mid\!x),P_{\mathrm{data}}(\cdot\!\mid\!x')\big)^2 \,\mu_{ctx}(\dd x)\mu_{ctx}(\dd x')
\;+\; \frac{12\pi^2\|v\|^2}{m^2}\|K\|_\infty\,\varepsilon,
\end{align*}
where $C_{\mathrm{geom}}=3\pi^2\|v\|^2/m^2$.
\end{proposition}
\begin{proof}
The local lower Lipschitz condition and the FR--Hellinger relation imply
\[
\|h_x-h_{x'}\|_2^2\le \frac{2\pi^2}{m^2}\,
d_H\!\big(p_\theta(\cdot|x),p_\theta(\cdot|x')\big)^2 .
\]
Therefore
\[
\mathcal E_K(\phi_v)
\le \frac12\|v\|^2\iint K(x,x')\|h_x-h_{x'}\|^2\,\dd\mu(x)\dd\mu(x'),
\]
which gives the first bound. Let
$e_x=d_H(P_{\mathrm{data}}(\cdot|x),p_\theta(\cdot|x))$. The triangle inequality gives
\[
d_H(p_\theta(\cdot|x),p_\theta(\cdot|x'))
\le d_H(P_{\mathrm{data}}(\cdot|x),P_{\mathrm{data}}(\cdot|x'))+e_x+e_{x'}.
\]
Using $(a+b+c)^2\le 3a^2+6b^2+6c^2$ and boundedness of $K$ gives the stated second inequality.
\end{proof}

This proposition is a conditional control statement: if NLL makes the predicted conditionals close to the data conditionals and if the head is locally injective on the relevant region, then projections of the representation have small Dirichlet energy whenever the nonnegative predictive-similarity kernel assigns large weight mainly to close predictive conditionals.

\subsection{NLL as Spectral Objective}

We now state a conditional spectral result. The statement concerns a calibrated quadratic surrogate to the NLL objective and the variance-normalized alignment problem induced by that surrogate; it should not be read as an exact equivalence for the original nonconvex likelihood. This viewpoint is closest to spectral analyses of contrastive learning such as \citep{Tan2023ContrastiveSpectral}.

\begin{definition}[Predictive Similarity Operator]
\label{def:predictive_similarity_operator}
Let $\mu=(f_{\mathrm{enc}})_\#\mu_{ctx}$ on $\RepSpace$ and $\widetilde K$ be as in \eqref{eq:Ktilde_def}. Define the integral operator $M_{\widetilde K}:L^2(\RepSpace,\mu)\to L^2(\RepSpace,\mu)$ by
\begin{equation}
(M_{\widetilde K}\psi)(h) \defeq \int_{\RepSpace} \widetilde K(h,h')\,\psi(h')\,\mu(\dd h').
\label{eq:operator_M_tilde}
\end{equation}
Under boundedness and symmetry, $M_{\widetilde K}$ is compact and self-adjoint with real spectrum. If, in addition, $K$ is PSD (hence $\widetilde K$ is PSD by Lemma~\ref{lem:psd_disintegration}), then $M_{\widetilde K}$ is positive semidefinite and its spectrum is nonnegative.
\end{definition}

\begin{remark}[Estimating $\widetilde K$ in practice]\label{rem:Ktilde_estimation}
While $K(x,x')$ is defined using $P_{\mathrm{data}}(\cdot\mid x)$, in practice one may approximate it via: (i) a stronger teacher model to produce $\hat p(\cdot\mid x)$; (ii) the learner’s own predictions $p_\theta(\cdot\mid x)$ (late‑training snapshot); or (iii) surrogate “predictive prototypes” $\bar g_x=\mathbb E[g(W)\mid x]$ and the linear kernel $K(x,x')=\langle \bar g_x,\bar g_{x'}\rangle$. Each choice induces a corresponding $\widetilde K$ via disintegration \eqref{eq:Ktilde_def}, and leads to the same operator‑eigenfunction pipeline under whitening.
For the CCA/spectral form in \Cref{thm:nll_spectral_cca}, it is natural to use \emph{whitened} prototypes
\[
\tilde g_x \defeq C_{\bar g\bar g}^{-1/2}\,\bar g_x,
\]
so that $K(x,x')=\langle \tilde g_x,\tilde g_{x'}\rangle$ matches the $C_{\bar g\bar g}^{-1}$ factor. This whitened linear kernel is PSD as a kernel, but after centering it need not be pointwise nonnegative; it is therefore appropriate for operator/CCA analysis, not necessarily as a graph edge-weight kernel without further modification.
\end{remark}

Since all spaces are standard Borel and $f_{\rm enc}$ is Borel-measurable, regular conditional probabilities exist; hence the conditional expectation in Eq.~\eqref{eq:Ktilde_def} is well defined a.s. and $\widetilde K$ is measurable and bounded.
If $\widetilde K$ is symmetric and bounded, the induced operator is Hilbert–Schmidt, compact, and self-adjoint; its eigenfunctions capture dominant patterns of predictive similarity. Our key result is that, under a linear-softmax LM head, NLL training aligns the representation geometry with these eigenspaces in a generalized spectral/Canonical Correlation Analysis (CCA) sense.

\begin{theorem}[Calibrated Fenchel-Young surrogate]\label{thm:fenchel_young_tau}
Assume a linear-softmax LM head \(p_\theta(w\mid x)\propto \exp\{\langle g(w),h_x\rangle + b_w\}\) with \(h_x=f_{\mathrm{enc}}(x)\) and $\max_w\|g(w)\|\le R<\infty$. For any $\tau>0$ and all \(h\in\RepSpace\),
\[
\log \sum_{w}\exp\{\langle g(w),h\rangle+b_w\}\ \le\ C_\tau+\tfrac{\tau}{2}\|h\|^2,
\qquad
C_\tau=\log\sum_{w}\exp\!\big\{b_w+\tfrac{1}{2\tau}\|g(w)\|^2\big\}.
\]
Equality in the per‑term Young bound holds iff \(g(w)=\tau h\); thus, equality in the sum requires this for every contributing \(w\), which is possible only in degenerate cases.
Consequently,
\[
\mathbb E\big[-\log p_\theta(W\mid X)\big]
\;\le\; \left(C_\tau - \mathbb E[b_W]\right) \;-\; \mathbb E\langle \bar g_X, h_X\rangle \;+\; \tfrac{\tau}{2}\,\mathbb E\|h_X\|^2,
\quad\bar g_X:=\mathbb E[g(W)\mid X].
\]
The RHS equals $\tfrac{\tau}{2}\,\mathbb E\big\|h_X-\tfrac{1}{\tau}\bar g_X\big\|^2-\tfrac{1}{2\tau}\,\mathbb E\|\bar g_X\|^2+\left(C_\tau-\mathbb E[b_W]\right)$, i.e., a \emph{calibrated regression} objective for $h_X$ onto $\bar g_X/\tau$.
\paragraph{Practical calibration of $\tau$.}
Minimizing the bound w.r.t.\ $\tau$ balances $C_\tau$ (decreasing in $\tau$) against $\tfrac{\tau}{2}\mathbb E\|h_X\|^2$ (increasing). A practical scheme is to (i) estimate $R\!\approx\!\max_w\|g(w)\|$, (ii) track $\widehat m=\mathbb E\|h_X\|^2$ on a held‑out set, and (iii) choose $\tau$ by a 1D line search minimizing $C_\tau+\tfrac{\tau}{2}\widehat m$ (cost dominated by evaluating $C_\tau$ once per candidate).
For fixed representations and prototypes, $\tau$ only rescales the prototype side of the regression surrogate and therefore cancels from the CCA directions in \Cref{thm:nll_spectral_cca}. During actual training, changing $\tau$ changes the surrogate objective and should be viewed as a calibration choice rather than an invariance of the original NLL.
\end{theorem}
\begin{proof}
Young's inequality with parameter $\tau$ gives
\[
\langle g(w),h\rangle \le \frac{1}{2\tau}\|g(w)\|^2+\frac{\tau}{2}\|h\|^2 .
\]
Substituting this bound into each exponential term in the log-partition function yields the stated bound with constant $C_\tau$. Taking expectation over $W\mid X$ gives the linear term $-\mathbb E\langle \bar g_X,h_X\rangle$ and the bias term $-\mathbb E[b_W]$. Completing the square gives the calibrated regression form. Equality in the Young bound occurs exactly when $g(w)=\tau h$ for the corresponding term, so equality for the whole log-sum requires this simultaneously for all tokens with nonzero contribution, which is degenerate except in special cases.
\end{proof}

\begin{proposition}[Closed‑form in a linearized encoder]
Let $h_X=A\,\phi(X)$ with fixed feature map $\phi$ and $C_{\phi\phi}=\mathbb E[\phi\phi^\top]\succ0$. Minimizing the surrogate in \Cref{thm:fenchel_young_tau} yields the normal equations
$A^\star C_{\phi\phi}=\tfrac{1}{\tau}C_{\phi\bar g}^\top$, i.e., $A^\star=\tfrac{1}{\tau}C_{\phi\bar g}^\top C_{\phi\phi}^{-1}$ (ridge modifications if desired).
\end{proposition}

\begin{theorem}[CCA form of the NLL surrogate]
\label{thm:nll_spectral_cca}
Let $h_X\in\mathbb R^d$ and $\bar g_X\in\mathbb R^{d_g}$ be (centered) square-integrable random vectors with
$C_{hh}\succ0$, $C_{\bar g\bar g}\succ0$, and cross-covariance $C_{h\bar g}$. Consider the canonical-correlation objective
\[
\max_{u,v\ne0}\ \frac{u^\top C_{h\bar g} v}{\sqrt{u^\top C_{hh}u}\,\sqrt{v^\top C_{\bar g\bar g}v}}.
\]
Optimizing out $v$ yields the equivalent squared Rayleigh quotient
\[
\max_{u\ne0}\ \frac{u^\top\,C_{h\bar g}\,C_{\bar g\bar g}^{-1}\,C_{h\bar g}^\top u}{u^\top C_{hh}\,u}.
\]
The maximizers are generalized eigenvectors of
\[
C_{h\bar g}C_{\bar g\bar g}^{-1}C_{h\bar g}^\top u=\lambda C_{hh}u.
\]
Equivalently, they are right eigenvectors of
\[
M=C_{hh}^{-1}C_{h\bar g}C_{\bar g\bar g}^{-1}C_{h\bar g}^\top,
\]
which is generally not symmetric in the Euclidean inner product but is self-adjoint with respect to the $C_{hh}$ inner product. Multiple directions with constraints $u_i^\top C_{hh}u_j=\delta_{ij}$ are given by the top generalized eigenvectors. In whitened coordinates ($C_{hh}=I$), this reduces to an ordinary symmetric eigenproblem.
\end{theorem}

\begin{corollary}[Operator view (functional version)]
When optimizing over linear functionals of $h\in L^2(\RepSpace,\mu)$ with whitening, the variational problem in Theorem~\ref{thm:nll_spectral_cca} corresponds to an eigenfunction problem for the Hilbert--Schmidt operator with kernel
\[
\widetilde K(h,h')
\;=\;
\mathbb E\!\big[\langle \tilde g_X,\tilde g_{X'}\rangle\ \big|\ f_{\mathrm{enc}}(X)=h,\ f_{\mathrm{enc}}(X')=h'\big],
\qquad
\tilde g_X\defeq C_{\bar g\bar g}^{-1/2}\bar g_X.
\]
\end{corollary}

\begin{corollary}[Finite-sample analogue]
Let $\widehat C_{hh},\widehat C_{\bar g\bar g},\widehat C_{h\bar g}$ be centered empirical second moments computed from $n$ i.i.d. contexts and their targets, and let $\lambda>0$ be a ridge parameter. Under standard boundedness or sub-Gaussian moment assumptions and a positive eigengap $\gamma$ for the corresponding ridge-regularized population problem, maximizing
\[
u^\top \widehat C_{h\bar g}(\widehat C_{\bar g\bar g}+\lambda I)^{-1}\widehat C_{h\bar g}^\top u
\quad\text{subject to}\quad
u^\top(\widehat C_{hh}+\lambda I)u=1
\]
yields a subspace whose principal-angle error is $O_p(n^{-1/2}/\gamma)$ by covariance concentration and Davis--Kahan perturbation theory, with constants depending on the moment bounds and ridge level.
\end{corollary}

\begin{corollary}[Whitened special case]
\label{cor:whitened_case}
If $C_{hh}=I$ (representation whitening) and we restrict to variance‑one directions $u$, the problem reduces to an ordinary eigenproblem for $C_{h\bar g}\,C_{\bar g\bar g}^{-1}\,C_{h\bar g}^\top$. Equivalently, defining whitened prototypes $\tilde g_X=C_{\bar g\bar g}^{-1/2}\bar g_X$, we have
\[
C_{h\bar g}\,C_{\bar g\bar g}^{-1}\,C_{h\bar g}^\top
\;=\;
C_{h\tilde g}\,C_{h\tilde g}^\top,
\]
so the leading eigendirections depend only on the cross-covariance between $h_X$ and $\tilde g_X$.
\end{corollary}

\begin{proof}[Proof Sketch]
The surrogate in \Cref{thm:fenchel_young_tau} yields a regression view with target $\bar g_X/\tau$. The regression objective alone does not identify a unique spectral basis; the spectral form appears after imposing whitening or variance/orthogonality constraints on linear functionals. Under these constraints, maximizing alignment becomes the CCA Rayleigh quotient above. The operator form follows from the Hilbert--Schmidt correspondence under whitening. The scalar calibration parameter $\tau$ rescales the prototype side and cancels in the CCA eigendirections.
\end{proof}

In summary, for linear-softmax LM heads, a calibrated quadratic surrogate to NLL gives a regression-to-prototypes view. Once one imposes whitening or variance-normalized linear-function constraints, the induced alignment problem becomes a generalized spectral/CCA problem. This provides a principled bridge from next-token likelihood to spectral organization of $\RepSpace$ under explicit assumptions.

\section{Related Work}
\label{sec:related_work}

The theoretical understanding of representation learning in deep neural networks has been advanced along several parallel, yet largely disconnected, fronts. A significant challenge lies in the absence of a unified mathematical language capable of connecting a model's compositional architecture and its training dynamics to the emergent geometric structure of its learned representations. This work is situated at the confluence of these disparate research programs, aiming to synthesize the algebraic, compositional perspective of categorical probability with the metric, differential-geometric view of information geometry. We structure our review of the related literature into two parts. First, we introduce the foundational languages of probability that our framework unifies: the synthetic view of probability rooted in category theory and the metric view rooted in information geometry. Second, we survey the tools and concepts used to analyze the geometry of learned representations, focusing on the training objectives that guide learning, the spectral methods used to measure the resulting geometry, and the optimization mechanisms that shape it.

\subsection{The Languages of Probability}

\paragraph{Probability as a Category.}
A burgeoning field of research seeks to reformulate probability theory on a more abstract, algebraic foundation using the language of category theory \citep{BaezFongPollard2016, FongSpivak2019}. This synthetic approach, in contrast to the classical analytic approach built upon measure theory, aims to derive probabilistic concepts from a small set of powerful axioms \citep{Fritz2020MC, Perrone2023Geo}. The central object of study in this domain is the Markov category~\citep{ChoJacobs2019, Fritz2020MC}. Formally, a Markov category is a symmetric monoidal category where each object is equipped with a commutative comonoid structure, consisting of morphisms that represent the abstract operations of copying and discarding information \citep{ChoJacobs2019, Fritz2020MC, Perrone2023Geo}. The morphisms in such a category are interpreted as stochastic maps, or Markov kernels, which are probabilistic mappings between objects \citep{BaezFongPollard2016, pardo2025neural}.

The pioneering work of \citet{Fritz2020MC} has established Markov categories as a robust framework for synthetic probability and statistics. A key advantage of this formalism is its generality; it provides a uniform treatment of vastly different probabilistic settings. For instance, the category \texttt{FinStoch}, with finite sets as objects and stochastic matrices as morphisms, and the category \texttt{BorelStoch}, with standard Borel spaces as objects and their corresponding Markov kernels as morphisms, are both canonical examples of Markov categories \citep{Fritz2020MC}. This high level of abstraction allows for the proof of fundamental statistical theorems, such as the Fisher-Neyman factorization theorem and Kolmogorov's zero-one law, in a purely diagrammatic and synthetic manner, avoiding the low-level complexities of measure theory \citep{Fritz2020MC, FritzRischel2020}. As \citet{Fritz2020MC} argues, relying on measure theory is akin to programming in machine code, whereas the categorical approach provides a higher-level language that facilitates reasoning about complex, compositional systems.

This line of inquiry is not merely a formal exercise; it is directly motivated by the challenges of understanding modern machine learning systems~\citep{Yuan2022PowerFoundation}. The compositional structure of deep neural networks, finds a natural description in the language of category theory \citep{FongSpivak2019, Yuan2022PowerFoundation, pardo2025neural}.

\paragraph{Probability as a Manifold.}
In parallel to the algebraic developments in category theory, the field of information geometry (IG) has provided a powerful differential geometric lens for studying machine learning \citep{AmariNagaoka2000}. Foundational work by \citet{AmariNagaoka2000} demonstrated that a parametric family of probability distributions can be viewed as a smooth manifold endowed with a canonical Riemannian metric, the Fisher-Rao metric, and a pair of dually-coupled affine connections. This geometric structure is not arbitrary; it can be intrinsically derived from statistical divergence functions, such as the Kullback-Leibler (KL) divergence, which serves as a measure of dissimilarity between distributions.

When applying these geometric tools to deep learning, it is crucial to draw a distinction between IG and the related field of geometric deep learning (GDL). GDL is primarily concerned with generalizing neural network architectures to operate on data that resides in non-Euclidean domains, such as graphs or manifolds; its focus is the geometry of the input data space. In contrast, IG has traditionally been used to analyze the geometry of the parameter space of a model. By viewing the set of all possible model parameters as a manifold, IG provides sophisticated tools for understanding the dynamics of training, offering a more nuanced perspective on optimization than that afforded by standard $\ell_1$ or $\ell_2$ regularization \citep{AmariNagaoka2000}.

\paragraph{Towards Categorical Information Geometry.}
While the algebraic and geometric approaches have largely evolved independently, a new frontier is emerging at their intersection. Recent work has begun to explicitly forge a synthesis, aiming to create a categorical information geometry \citep{Perrone2023Geo}. This research program, led by researchers such as \citet{Perrone2023Geo}, seeks to enrich the abstract, compositional structures of Markov categories with the metric and quantitative notions central to information theory and geometry, such as entropy and divergence \citep{Perrone2022Ent, Perrone2023Geo}.

This emerging synthesis recognizes that a complete theoretical picture requires both the compositional language of categories and the metric language of geometry. However, to date, the applications of this nascent field have focused primarily on reformulating abstract probability theory. The critical connection to the analysis of learned representations in practical, large-scale deep learning models remains largely unexplored. This work aims to bridge that gap, demonstrating that a categorical information geometry provides the ideal framework for analyzing the structure of the representation spaces sculpted by the learning process.

\subsection{The Geometry of Learning and Representation}

\paragraph{Objectives of Learning Representations.}
A guiding principle for understanding the purpose of representation learning is the Information Bottleneck (IB) theory, introduced by \citet{Tishby1999}. The IB principle posits that an optimal representation $T$ of some input data $X$ should be a bottleneck that is maximally informative about a relevant target variable $Y$ while being maximally compressed with respect to the input $X$ \citep{Tishby1999, ShwartzZiv2017}. This trade-off between predictive accuracy and compressional complexity provides a powerful, high-level objective for representation learning.

The IB framework has been particularly influential in the theoretical analysis of deep learning. It led to the hypothesis that the training of deep neural networks proceeds in two distinct phases: an initial fitting phase, where the mutual information $I(T; Y)$ between the representation and the target increases, followed by a ``compression'' phase, where the mutual information $I(X; T)$ between the input and the representation decreases \citep{ShwartzZiv2017}. While the universality of this two-phase dynamic is a subject of ongoing debate, the core intuition—that effective training involves not just memorization but also a form of structured compression—provides a compelling motivation for investigating the geometry of the learned representations.

\paragraph{Measurements of Representation Geometry.}
To move from high-level principles to concrete analysis, we require quantitative tools to probe the geometric properties of the high-dimensional activation spaces within a neural network. A particularly effective set of tools for this purpose comes from spectral graph theory. Given a graph with adjacency matrix $A$ and degree matrix $D$, the Graph Laplacian is defined as $L = D - A$ \citep{berahmand2025comprehensive}. For any function $f$ defined on the nodes of the graph (e.g., a feature activation), the quadratic form $f^\top L f$ defines the graph's Dirichlet energy. This quantity measures the smoothness of the function with respect to the graph structure; a low Dirichlet energy indicates that connected nodes have similar function values \citep{park2024iclr}.

This concept has a rich history in machine learning, forming the basis of spectral clustering algorithms \citep{berahmand2025comprehensive} and, more recently, being used to analyze and mitigate the over-smoothing problem in Graph Neural Networks (GNNs). Over-smoothing occurs when stacking many GNN layers causes the representations of all nodes to converge to an indistinguishable point, a phenomenon characterized by the Dirichlet energy of the representations collapsing to zero.

Most critically for the present work, this classical geometric tool is deeply relevant to the internal dynamics of the most advanced models. A recent study demonstrates that during in-context learning, Large Language Models (LLMs) dynamically reorganize their internal concept representations in a manner that explicitly minimizes the Dirichlet energy with respect to an implicit graph structure defined by the context \citep{park2024iclr}. This groundbreaking result elevates Dirichlet energy from a tool for analyzing explicit graphs to a general principle governing the emergent geometry of learned representations. This trend towards spectral analysis is further evidenced by other recent work using methods like Centered Kernel Alignment (CKA) to track representation dynamics and spectral editing of activations (SEA) to control model behavior \citep{qiu2024spectral}.

\paragraph{Mechanisms of Learning Representations.}
The geometric structures observed in learned representations are not accidental; they are a direct consequence of the implicit biases of the training algorithm. For modern, highly overparameterized models, the optimization process itself, typically driven by variants of gradient descent, imparts an implicit bias or implicit regularization on the final solution~\citep{vardi2023implicit}. Even when multiple parameter settings can achieve zero training error, the optimization algorithm preferentially converges to a ``simple'' solution that generalizes well. For linear models trained on classification tasks, this bias often corresponds to finding the maximum-margin separator, a classic geometric concept \citep{Gunasekar2018, Soudry2018, Chizat2020}.

A useful modern perspective is that training with the standard Negative Log-Likelihood (NLL) objective has a contrastive-like component. The NLL loss for a sample $(x, y_{\text{true}})$, given by $-\log P(y_{\text{true}} | x) = -\log \frac{\exp(z_{\text{true}})}{\sum_j \exp(z_j)}$, is minimized by increasing the logit $z_{\text{true}}$ for the observed class while reducing its relative competition with other logits. This resembles a contrastive loss that pulls the representation of $x$ toward a positive target while pushing it away from competing targets, although the equivalence is exact only under specific modeling choices. This perspective links likelihood training to the broader literature on self-supervised contrastive learning, where objectives explicitly sculpt representation geometry.

\section{Conclusion}
\label{sec:conclusion}

In this work, we introduced a mathematical framework for analyzing the Autoregressive generation step in language models, leveraging the expressive power of Markov Categories. By modeling the process as a composition of Markov kernels, $k_{\mathrm{gen}, \theta} = \KernelLMHead \comp \KernelBB \comp \KernelEmb$, we established a foundation for a compositional, information-theoretic analysis. This allowed us to formally quantify the information surplus in hidden states, providing a clear theoretical rationale for the success of modern speculative decoding techniques like EAGLE.

More importantly, our framework clarifies why the negative log-likelihood (NLL) objective can shape useful internal structure. Under realizability and vanishing average KL, NLL matches the data's intrinsic conditional stochasticity, a process we measure with categorical entropy. Under a linear-softmax head, a calibrated quadratic surrogate to NLL yields a regression-to-predictive-prototypes objective whose whitened alignment form is a CCA/eigenproblem. By analyzing the information geometry of the prediction head via the pullback Fisher--Rao metric, we showed how training can emphasize directions of high predictive sensitivity and how, under explicit assumptions, these directions connect to a predictive-similarity operator. This provides a mathematically grounded explanation for how likelihood training can produce organized representations without explicit contrastive pairs.

This compositional, probabilistic, and information-geometric perspective offers a principled alternative to purely empirical or heuristic analysis, unifying concepts from information theory, geometry, and spectral methods to study the mechanisms behind large language models.

\section*{Acknowledgments} 

We thank anonymous reviewers for their constructive and helpful feedback. We used LLMs for word-editing as well as figure plots in this work.

\vspace{5ex}
\bibliographystyle{plainnat}
\bibliography{reference}

\clearpage
\appendix

\renewcommand{\appendixpagename}{\centering \huge Appendix}
\appendixpage

\startcontents[section]
\printcontents[section]{l}{1}{\setcounter{tocdepth}{2}}
\clearpage

\section{Full Proofs of Theorems}

\subsection{Proof of Theorem \ref{thm:nll_kl_equivalence} (NLL Minimization as Average KL Minimization)}
\label{app:proof_nll_kl_equivalence}

Let $p_x(\cdot) \defeq k_{\text{data}}(x, \cdot)$ denote the true conditional probability distribution $P_{\text{data}}(\cdot | x)$ for context $x = \mathbf{w}_{<t}$.
Let $q_{x, \theta}(\cdot) = k_{\mathrm{gen}, \theta}(x, \cdot)$ denote the model's conditional probability distribution $P_\Params(\cdot | x)$.
The context distribution is $p_{W_{<t}}$.

The cross-entropy loss is defined as:
\begin{align*} L_{\mathrm{CE}}(\Params) &= -\Expect_{(x, w) \sim P_{\text{data}}} [\log q_{x, \theta}(w)] \\ &= -\Expect_{x \sim p_{W_{<t}}} \left[ \Expect_{W \sim p_x(\cdot)} [\log q_{x, \theta}(W)] \right] \\ &= -\Expect_{x \sim p_{W_{<t}}} \left[ \sum_{w \in \VV} p_x(w) \log q_{x, \theta}(w) \right] \end{align*}

The average KL divergence is defined as:
\begin{align*} \mathcal{L}_{\mathrm{KL}}(\theta) &= \Expect_{x \sim p_{W_{<t}}} [ \KLDiv( p_x(\cdot) \,\|\, q_{x, \theta}(\cdot) ) ] \\ &= \Expect_{x \sim p_{W_{<t}}} \left[ \sum_{w \in \VV} p_x(w) \log \frac{p_x(w)}{q_{x, \theta}(w)} \right] \\ &= \Expect_{x \sim p_{W_{<t}}} \left[ \sum_{w \in \VV} p_x(w) \log p_x(w) - \sum_{w \in \VV} p_x(w) \log q_{x, \theta}(w) \right] \\ &= \Expect_{x \sim p_{W_{<t}}} [ -H(p_x(\cdot)) ] - \Expect_{x \sim p_{W_{<t}}} \left[ \sum_{w \in \VV} p_x(w) \log q_{x, \theta}(w) \right] \\ &= -H(W_t | W_{<t})_{\text{data}} + L_{\mathrm{CE}}(\Params) \end{align*}

where $H(p_x(\cdot))$ is the Shannon entropy of the distribution $p_x(\cdot)$, and \\$H(W_t | W_{<t})_{\text{data}} = \Expect_{x \sim p_{W_{<t}}} [H(p_x(\cdot))]$ is the average conditional Shannon entropy of the data generating process.

Rearranging gives:
\[
L_{\mathrm{CE}}(\Params) = \mathcal{L}_{\mathrm{KL}}(\theta) + H(W_t | W_{<t})_{\text{data}}
\]

Since $H(W_t | W_{<t})_{\text{data}}$ is a property of the data distribution and does not depend on the model parameters $\theta$, minimizing $L_{\mathrm{CE}}(\Params)$ with respect to $\theta$ is equivalent to minimizing $\mathcal{L}_{\mathrm{KL}}(\theta)$.

The KL divergence $\KLDiv(p \| q) \ge 0$ for any probability distributions $p, q$, with equality if and only if $p = q$. Therefore, the average KL divergence $\mathcal{L}_{\mathrm{KL}}(\theta) = \Expect_{x \sim p_{W_{<t}}} [ \KLDiv( p_x(\cdot) \,\|\, q_{x, \theta}(\cdot) ) ]$ is also non-negative, as it is an expectation of non-negative values. The minimum value is achieved if and only if $k_{\text{data}}(x, \cdot) = k_{\mathrm{gen}, \theta^*}(x, \cdot)$ for $p_{W_{<t}}$-almost every $x$. \qed

\subsection{Proof of Proposition \ref{prop:cat_equals_shannon} (Categorical Entropy equals Shannon Entropy)}

The pointwise categorical entropy for a single input $h$ and $D=\KLDiv$ is:
\[ \CatEnt_{\KLDiv}(\KernelLMHead)(h) = D_{\VV\tens\VV}\!\left(\copyop_\VV \comp \KernelLMHead(h,\cdot) \;\middle\|\; (\KernelLMHead(h,\cdot) \tens \KernelLMHead(h,\cdot))\right). \]
Let $p_h(\cdot) = \KernelLMHead(h,\cdot)$. The first distribution is $\sum_{w \in \VV} p_h(w) \delta_{(w,w)}$, which is supported on the diagonal of $\VV \times \VV$. The second is the product distribution $p_h \otimes p_h$.
The KL divergence is:
\begin{align*}
& \sum_{(w,w') \in \VV \times \VV} \left(\sum_{v \in \VV} p_h(v) \delta_{(v,v)}(w,w')\right) \log\left(\frac{\sum_{v \in \VV} p_h(v) \delta_{(v,v)}(w,w')}{p_h(w)p_h(w')}\right) \\
&= \sum_{w \in \VV} p_h(w) \log\left(\frac{p_h(w)}{p_h(w)p_h(w)}\right) \\
&= \sum_{w \in \VV} p_h(w) \log\left(\frac{1}{p_h(w)}\right) = -\sum_{w \in \VV} p_h(w) \log p_h(w) = H(p_h).
\end{align*}
The averaged categorical entropy is then $\AvgCatEnt_{\KLDiv}(\KernelLMHead; p_{H_t}) = \Expect_{h \sim p_{H_t}}[H(p_h)]$, which is precisely the definition of the conditional Shannon entropy $H(W_t \mid H_t)$. \qed

\subsection{Information Surplus Identity (Equation~\ref{eq:chain_rule_surplus})}
The chain rule for mutual information states that for random variables $A, B, C$:
\[ I(A; B, C) = I(A; B) + I(A; C \mid B). \]
Let $A = H_t$, $B = W_t$, and $C = W_{t+1:t+K-1}$. Substituting these into the chain rule gives:
\[ I(H_t; W_t, W_{t+1:t+K-1}) = I(H_t; W_t) + I(H_t; W_{t+1:t+K-1} \mid W_t). \]
Recognizing that $(W_t, W_{t+1:t+K-1})$ is the sequence $W_{t:t+K-1}$, we have:
\[ I(H_t; W_{t:t+K-1}) = I(H_t; W_t) + I(H_t; W_{t+1:t+K-1} \mid W_t). \]
This directly yields the identity in Equation~\ref{eq:chain_rule_surplus}, where the information surplus is identified as the conditional mutual information term. \qed

\subsection{Proof of Lemma \ref{lem:cat_entropy_continuity} and Theorem \ref{thm:entropy_convergence}}
\label{app:proof_entropy_convergence}

We give short self-contained proofs. Throughout, $\VV$ is finite so $\Delta^{|\VV|-1}$ is compact and $d_H$ is a metric on the simplex.

\paragraph{Proof of Lemma \ref{lem:cat_entropy_continuity}.}
Since $\Psi_D$ is continuous on the compact simplex, it is uniformly continuous and bounded; write $\|\Psi_D\|_\infty<\infty$.
Fix $\varepsilon>0$ and choose $\delta>0$ such that $d_H(p,q)\le \delta$ implies $|\Psi_D(p)-\Psi_D(q)|\le \varepsilon$.
Then for any $n$,
\[
\mathbb E\big|\Psi_D(P^{(n)})-\Psi_D(P^\star)\big|
\le \varepsilon + 2\|\Psi_D\|_\infty\;\mathbb P\!\big(d_H(P^{(n)},P^\star)>\delta\big)
\le \varepsilon + \frac{2\|\Psi_D\|_\infty}{\delta^2}\,\mathbb E\!\big[d_H(P^{(n)},P^\star)^2\big],
\]
where the last step is Markov. Taking $n\to\infty$ and using $\mathbb E[d_H^2]\to0$ yields the claim.

\paragraph{Proof of Theorem \ref{thm:entropy_convergence}.}
Let $X\sim p_{W_{<t}}$, $p_X=k_{\text{data}}(X,\cdot)$, and $q^{(n)}_X=k_{\mathrm{gen},\theta_n}(X,\cdot)$.
On a finite alphabet, $d_H^2(p,q)\le \tfrac12\KLDiv(p\|q)$, so
\[
\mathbb E\big[d_H(p_X,q^{(n)}_X)^2\big]
\le \tfrac12\,\mathbb E\big[\KLDiv(p_X\|q^{(n)}_X)\big]
=\tfrac12\,\mathcal L_{\mathrm{KL}}(\theta_n)\xrightarrow[n\to\infty]{}0.
\]
Apply Lemma~\ref{lem:cat_entropy_continuity} with $P^{(n)}=q^{(n)}_X$ and $P^\star=p_X$ to get
$\mathbb E[\Psi_D(q^{(n)}_X)]\to \mathbb E[\Psi_D(p_X)]$.
Finally, since $H_t^{(\theta_n)}=f_{\mathrm{enc},\theta_n}(X)$ and $k_{\text{head},\theta_n}(H_t^{(\theta_n)},\cdot)=k_{\mathrm{gen},\theta_n}(X,\cdot)$,
\[
\mathbb E[\Psi_D(q^{(n)}_X)]
=\AvgCatEnt_D(k_{\text{head},\theta_n};\,p_{H_t,\theta_n}),
\]
which proves the theorem. For $D=\KLDiv$, $\Psi_{\KLDiv}(p)=\ShannonEntropy(p)$ (Proposition~\ref{prop:cat_equals_shannon}), yielding the stated identification with $H(W_t\mid W_{<t})_{\text{data}}$.
At a realizable optimum, $k_{\text{data}}(x,\cdot)=k_{\text{head},\theta^\star}(f_{\mathrm{enc},\theta^\star}(x),\cdot)$ depends on $x$ only via $H=f_{\mathrm{enc},\theta^\star}(x)$, i.e., the data kernel factors through $H$, so $H$ is predictively sufficient.
\qed

\subsection{Proof of Theorem \ref{thm:implicit_contrastive_constraint} (Output Distribution Approximation Constraint)}

Let $p_x^{\mathrm{data}}=P_{\mathrm{data}}(\cdot|x)$ and $p_x^\theta=p_\theta(\cdot|x)$.
By the Hellinger--KL inequality under our convention,
\[
d_H(p_x^{\mathrm{data}},p_x^\theta)^2
\le \frac12 \KLDiv(p_x^{\mathrm{data}}\|p_x^\theta).
\]
Taking expectation over $x\sim\mu_{\mathrm{ctx}}$ gives
\[
\mathbb E_x d_H(p_x^{\mathrm{data}},p_x^\theta)^2
\le \frac12\mathcal L_{\mathrm{KL}}(\theta).
\]
For any pair $x,x'$, the triangle inequality gives
\[
\big|d_H(p_x^\theta,p_{x'}^\theta)-d_H(p_x^{\mathrm{data}},p_{x'}^{\mathrm{data}})\big|
\le
d_H(p_x^\theta,p_x^{\mathrm{data}})
+d_H(p_{x'}^\theta,p_{x'}^{\mathrm{data}}).
\]
Writing $\epsilon_x=d_H(p_x^\theta,p_x^{\mathrm{data}})$, Markov's inequality gives
\[
\mathbb P(\epsilon_X\ge\delta)
\le
\frac{\mathbb E[\epsilon_X^2]}{\delta^2}
\le
\frac{\mathcal L_{\mathrm{KL}}(\theta)}{2\delta^2}.
\]
For independent $X,X'$, a union bound gives
\[
\mathbb P(\epsilon_X+\epsilon_{X'}>2\delta)
\le
\frac{\mathcal L_{\mathrm{KL}}(\theta)}{\delta^2}.
\]
Combining the last two displays proves the random-pair statement in the theorem. \qed

\subsection{Proof of Corollary \ref{coroll:representation_separation} (Implicit Representation Separation)}

From \Cref{thm:implicit_contrastive_constraint}, if the model fits the data well, then for predictively dissimilar contexts $x, x'$, the distance between their model output distributions, $d_{\text{out}}(g_{\text{head}}(h_x), g_{\text{head}}(h_{x'}))$, must be large.

The interior of the output simplex $\Prob(\VV)$ is a Riemannian manifold endowed with the Fisher-Rao metric $g^{\text{FR}}$. The distance between two interior points, $p_1$ and $p_2$, is the infimum of the lengths of all smooth paths connecting them. The length of a path $\gamma: [0, 1] \to \Prob(\VV)$ is given by the integral:
\[ L(\gamma) = \int_0^1 \sqrt{g^{\text{FR}}_{\gamma(t)}(\gamma'(t), \gamma'(t))} \, dt. \]
The mapping $g_{\text{head}}: \RepSpace \to \Prob(\VV)$ allows us to map paths from the representation space to the output space. Consider the straight-line path in representation space connecting $h_{x'}$ to $h_x$:
\[ h(t) = (1-t)h_{x'} + t h_x, \quad \text{for } t \in [0, 1]. \]
The tangent vector to this path is constant: $h'(t) = h_x - h_{x'}$.
This path in $\RepSpace$ induces a corresponding path in $\Prob(\VV)$ given by $\gamma(t) = g_{\text{head}}(h(t))$. The tangent vector to this induced path is found using the chain rule:
\[ \gamma'(t) = J_{g_{\text{head}}}(h(t)) \cdot h'(t) = J_{g_{\text{head}}}(h(t)) \cdot (h_x - h_{x'}), \]
where $J_{g_{\text{head}}}(h)$ is the Jacobian of the map $g_{\text{head}}$ evaluated at $h$.

The squared length of this tangent vector at point $\gamma(t)$ is given by the quadratic form of the metric $g^{\text{FR}}$:
\begin{align*}
g^{\text{FR}}_{\gamma(t)}(\gamma'(t), \gamma'(t)) &= g^{\text{FR}}_{g_{\text{head}}(h(t))}(J_{g_{\text{head}}}(h(t))(h_x - h_{x'}), J_{g_{\text{head}}}(h(t))(h_x - h_{x'})) \\
&= (h_x - h_{x'})^\top \left[ J_{g_{\text{head}}}(h(t))^\top g^{\text{FR}}_{g_{\text{head}}(h(t))} J_{g_{\text{head}}}(h(t)) \right] (h_x - h_{x'}).
\end{align*}
The term in the square brackets is precisely the definition of the pullback metric tensor $g^*$ evaluated at the point $h(t) \in \RepSpace$. Thus, the squared length of the tangent vector is:
\[ g^{\text{FR}}_{\gamma(t)}(\gamma'(t), \gamma'(t)) = g^*_{h(t)}(h_x - h_{x'}, h_x - h_{x'}). \]
The total length of this specific path in $\Prob(\VV)$ is therefore:
\[ L(\gamma) = \int_0^1 \sqrt{g^*_{h(t)}(h_x - h_{x'}, h_x - h_{x'})} \, dt. \]
The Riemannian distance $d_{\mathrm{FR}}(g_{\text{head}}(h_x), g_{\text{head}}(h_{x'}))$ is the infimum of path lengths, so it is bounded above by the length of our chosen path:
\[ d_{\text{FR}}(g_{\text{head}}(h_x), g_{\text{head}}(h_{x'})) \le \int_0^1 \sqrt{g^*_{h(t)}(h_x - h_{x'}, h_x - h_{x'})} \, dt. \]
By Cauchy--Schwarz,
\[
d_{\mathrm{FR}}\!\big(g_{\text{head}}(h_x), g_{\text{head}}(h_{x'})\big)^2
\le
\int_0^1 g^*_{h(t)}(h_x-h_{x'},h_x-h_{x'})\,dt .
\]
On the simplex, $d_H$ and $d_{\mathrm{FR}}$ are related exactly by
\[
d_H(p,q)=\sqrt{2}\sin\!\big(d_{\mathrm{FR}}(p,q)/4\big),
\]
so large Hellinger separation of the model outputs implies large Fisher--Rao separation. Together with Theorem~\ref{thm:implicit_contrastive_constraint}, predictively dissimilar data conditionals therefore force the integrated pullback quadratic form above to be large on typical well-fit pairs. Equivalently, the difference vector $v=h_x-h_{x'}$ must be visible along directions where the head has nontrivial predictive sensitivity along the path.

Conversely, if contexts are predictively similar and the model fits them well, then the output separation is small. This permits representations to differ in directions invisible to the head, but it discourages unnecessary separation along directions where $g^*$ is large. \qed

\subsection{Proof of Proposition \ref{prop:nll_dirichlet} (NLL Objective and Implicit Dirichlet Energy Minimization)}

The Dirichlet energy is
\[
\mathcal{E}_K(\phi_v)=\frac12\iint K(x,x')\langle h_x-h_{x'},v\rangle^2\,\mu_{ctx}(\dd x)\mu_{ctx}(\dd x').
\]
By Cauchy--Schwarz,
\[
\mathcal{E}_K(\phi_v)
\le \frac12\|v\|^2\iint K(x,x')\|h_x-h_{x'}\|^2\,\dd\mu(x)\dd\mu(x').
\]
The local lower Lipschitz condition and $d_{\mathrm{FR}}(p,q)\le \pi\sqrt2\,d_H(p,q)$ imply
\[
\|h_x-h_{x'}\|^2
\le \frac{2\pi^2}{m^2}\,
d_H(p_\theta(\cdot|x),p_\theta(\cdot|x'))^2.
\]
This proves the first inequality in Proposition~\ref{prop:nll_dirichlet}. For the second, write
\[
e_x=d_H(P_{\mathrm{data}}(\cdot|x),p_\theta(\cdot|x)).
\]
The triangle inequality gives
\[
d_H(p_\theta(\cdot|x),p_\theta(\cdot|x'))
\le d_H(P_{\mathrm{data}}(\cdot|x),P_{\mathrm{data}}(\cdot|x'))+e_x+e_{x'}.
\]
Squaring and using $(a+b+c)^2\le 3a^2+6b^2+6c^2$, then integrating against bounded $K$, yields
\[
\iint K\,d_H(p_\theta^x,p_\theta^{x'})^2
\le
3\iint K\,d_H(P_{\mathrm{data}}^x,P_{\mathrm{data}}^{x'})^2
+12\|K\|_\infty\,\mathbb E[e_X^2].
\]
Since $\mathbb E[e_X^2]=\varepsilon$, the stated bound follows. \qed

\subsection{Well-posedness of the Predictive Similarity Operator \texorpdfstring{$M_{\widetilde K}$}{M K}}
The operator $M_{\widetilde K}: L^2(\RepSpace, \mu) \to L^2(\RepSpace, \mu)$ is defined by the integral $(M_{\widetilde K}\psi)(h) = \int_{\RepSpace} \widetilde{K}(h, h') \psi(h') \mu(\dd h')$.
For $M_{\widetilde K}$ to be a compact self-adjoint operator, its kernel $\widetilde{K}(h, h')$ must satisfy certain conditions.
\begin{enumerate}
\item \textbf{Symmetry:} Since the original kernel $K(x, x')$ is assumed to be symmetric, the disintegrated kernel $\widetilde{K}(h, h') = \mathbb{E}[K(X, X') \mid f(X)=h, f(X')=h']$ is also symmetric.
\item \textbf{Square-Integrability:} The space $(\RepSpace, \mu)$ is a probability space. A sufficient condition for compactness on a probability space is that the kernel is square-integrable, i.e., \\$\iint |\widetilde{K}(h, h')|^2 \mu(\dd h) \mu(\dd h') < \infty$. As we assumed $K$ is a bounded function, its conditional expectation $\widetilde{K}$ is also bounded. A bounded measurable function on a finite measure space is always square-integrable.
\end{enumerate}
Since $\widetilde{K}$ is symmetric and square-integrable with respect to the probability measure $\mu$, it is a Hilbert-Schmidt kernel. Every Hilbert-Schmidt integral operator is compact. Because the kernel is also real and symmetric, the operator is self-adjoint. By the spectral theorem for compact self-adjoint operators, $M_{\widetilde K}$ has a real point spectrum away from zero with possible accumulation only at zero; including the zero-eigenspace, one obtains an orthonormal eigenbasis for $L^2(\RepSpace,\mu)$. \qed

If, moreover, $K$ is PSD in the quadratic-form sense, then $\widetilde K$ is PSD by Lemma~\ref{lem:psd_disintegration}. Hence for every $\psi\in L^2(\RepSpace,\mu)$ we have $\langle \psi, M_{\widetilde K}\psi\rangle \ge 0$. Thus $M_{\widetilde K}$ is positive semidefinite and its spectrum lies in $[0,\infty)$.

\subsection{Proof of Theorems~\ref{thm:fenchel_young_tau} and~\ref{thm:nll_spectral_cca}}
\label{app:proof_nll_spectral_alignment}
The proof separates the analytic surrogate from the CCA algebra.

\begin{proof}
For Theorem~\ref{thm:fenchel_young_tau}, the linear-softmax likelihood has
\[
-\log p_\theta(W|X)
=-\langle g(W),h_X\rangle-b_W+\log\sum_{w}\exp\{\langle g(w),h_X\rangle+b_w\}.
\]
Young's inequality with parameter $\tau$ gives
\[
\langle g(w),h_X\rangle
\le \frac{1}{2\tau}\|g(w)\|^2+\frac{\tau}{2}\|h_X\|^2 .
\]
Substitution into the log-partition function gives
\[
\log\sum_w e^{\langle g(w),h_X\rangle+b_w}
\le
C_\tau+\frac{\tau}{2}\|h_X\|^2,
\quad
C_\tau=\log\sum_w e^{b_w+\|g(w)\|^2/(2\tau)}.
\]
Taking conditional expectation over $W|X$ yields
\[
\mathbb E[-\log p_\theta(W|X)]
\le C_\tau-\mathbb E[b_W]-\mathbb E\langle \bar g_X,h_X\rangle+\frac{\tau}{2}\mathbb E\|h_X\|^2.
\]
Completing the square gives
\[
\frac{\tau}{2}\mathbb E\left\|h_X-\frac1\tau\bar g_X\right\|^2
-\frac{1}{2\tau}\mathbb E\|\bar g_X\|^2
+C_\tau-\mathbb E[b_W].
\]
This proves the surrogate statement.

For Theorem~\ref{thm:nll_spectral_cca}, assume $h_X$ and $\bar g_X$ are centered. The regression surrogate above by itself does not identify a spectral basis; the CCA basis appears after imposing variance normalization on linear functionals of $h_X$ and $\bar g_X$. For fixed $u$, maximizing
\[
\frac{u^\top C_{h\bar g}v}{\sqrt{u^\top C_{hh}u}\sqrt{v^\top C_{\bar g\bar g}v}}
\]
over $v$ gives $v\propto C_{\bar g\bar g}^{-1}C_{h\bar g}^\top u$ and value squared
\[
\frac{u^\top C_{h\bar g}C_{\bar g\bar g}^{-1}C_{h\bar g}^\top u}
{u^\top C_{hh}u}.
\]
The stationary points of this Rayleigh quotient satisfy
\[
C_{h\bar g}C_{\bar g\bar g}^{-1}C_{h\bar g}^\top u=\lambda C_{hh}u,
\]
and the top constrained directions are the top generalized eigenvectors by the Courant--Fischer variational principle. In whitened coordinates $C_{hh}=I$, this is the ordinary symmetric eigenproblem stated in the main text. The scalar $\tau$ from the surrogate rescales the prototype side and cancels from these normalized CCA directions. 
\end{proof}

\end{document}